\documentclass[preprint,12pt,amsmath,amssymb,superscriptaddress,floatfix]{revtex4-2}

\usepackage[T1]{fontenc}
\usepackage{amsmath, amssymb, amsfonts, amsthm}
\usepackage{graphicx}
\graphicspath{{./}}
\usepackage{float}
\usepackage{hyperref}
\usepackage{cleveref}
\usepackage{booktabs}
\usepackage{multirow}
\usepackage{listings}
\usepackage{xcolor}
\usepackage{bm}
\usepackage{braket}
\usepackage{tikz}

\definecolor{codegreen}{rgb}{0,0.6,0}
\definecolor{codegray}{rgb}{0.5,0.5,0.5}
\definecolor{codepurple}{rgb}{0.58,0,0.82}
\definecolor{backcolour}{rgb}{0.97,0.97,0.97}

\lstdefinestyle{pythonstyle}{
    backgroundcolor=\color{backcolour},
    commentstyle=\color{codegreen},
    keywordstyle=\color{blue},
    numberstyle=\tiny\color{codegray},
    stringstyle=\color{codepurple},
    basicstyle=\ttfamily\small,
    breakatwhitespace=false,
    breaklines=true,
    captionpos=b,
    keepspaces=true,
    numbers=left,
    numbersep=5pt,
    showspaces=false,
    showstringspaces=false,
    showtabs=false,
    tabsize=2,
    language=Python
}
\definecolor{pfblue}{RGB}{0,82,155}
\definecolor{pforange}{RGB}{200,80,0}
\hypersetup{
    colorlinks=true,
    linkcolor=pfblue,
    citecolor=pforange,
    urlcolor=pfblue
}

\newcommand{\complex}{\mathbb{C}}
\newcommand{\unitary}{\mathrm{U}}

\newtheorem{definition}{Definition}[section]
\newtheorem{theorem}{Theorem}[section]
\newtheorem{proposition}{Proposition}[section]
\newtheorem{corollary}[theorem]{Corollary}
\newtheorem{remark}{Remark}[section]

\begin{document}


\title{PhasorFlow: A Python Library for Unit Circle Based Computing}

\author{Dibakar Sigdel}
\email{devdeep137@gmail.com}
\affiliation{Mindverse Computing LLC, Lynnwood, WA 98087}
\author{Namuna Panday}
\affiliation{Mindverse Computing LLC, Lynnwood, WA 98087}

\date{\today}

\begin{abstract}
We present PhasorFlow, an open-source Python library for computing on the $S^1$ unit circle.
Inputs are encoded as complex phasors $z=e^{i\phi}$ on the $N$-torus ($\mathbb{T}^N$); 
as computation proceeds through unitary wave-interference gates, global norm is preserved 
while components drift into $\mathbb{C}^N$, letting algorithms leverage continuous geometric 
gradients. PhasorFlow makes three contributions. First, we formalize the Phasor Circuit model
 ($N$ threads, $M$ gates) with a 22-gate library spanning standard-unitary, non-linear, 
neuromorphic, and encoding operations under full matrix-algebra simulation. 
Second, we introduce the Variational Phasor Circuit (VPC), a trainable phase-native classifier
analogous to variational quantum circuits. Third, we introduce the Phasor Transformer block 
and Large Phasor Model (LPM), replacing $QK^TV$ attention with a parameter-free DFT token-mixing layer.
We validate the framework on financial volatility detection, neuromorphic associative memory, 
neural binding, period finding, and algorithmic logic applications that are unique to 
the library. This positions unit-circle computing as a deterministic, lightweight paradigm on 
classical hardware. Available at \url{https://github.com/mindverse-computing/phasorflow}.
\end{abstract}


\maketitle



\section{Introduction}
\label{sec:introduction}

The landscape of computation has historically been defined by the mathematical structures upon which data is represented and transformed.
Classical digital computers operate on discrete bits---points on a zero-dimensional manifold $\{0, 1\}$.
Quantum computers leverage qubits, which reside in the complex projective space $\mathbb{CP}^1$, enabling true quantum superposition and non-local entanglement through unitary operations on exponentially scaling Hilbert spaces \cite{nielsen2010quantum}.
Between these extremes lies a rich hierarchy of computational manifolds that remain largely unexplored in the context of programmable classical computing frameworks.

We refer to this hierarchy as the \emph{Geometric Ladder of Computation}:
\begin{enumerate}
    \item \textbf{Bits}: Points on a discrete set $\{0,1\}$---zero-dimensional.
    \item \textbf{Phasors}: Points on the unit circle $S^1 \cong \unitary(1)$---one-dimensional continuous manifold.
    \item \textbf{Qubits}: Points in complex projective space $\mathbb{CP}^n$---parameterized by both amplitude and phase.
\end{enumerate}

We visualize these three primary computational paradigms in \Cref{fig:three_paradigms}.

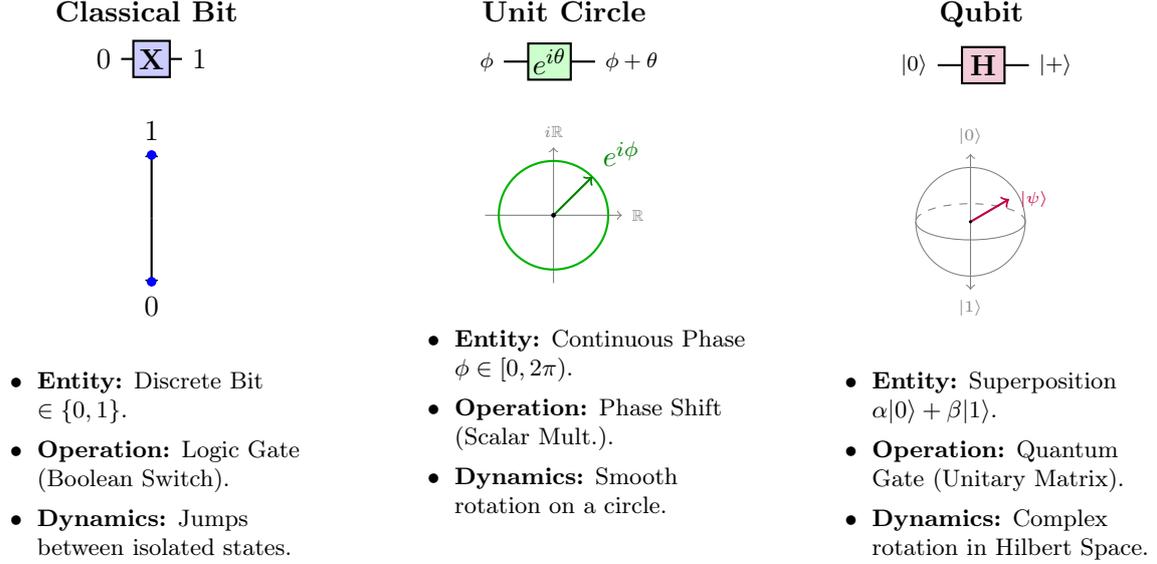
\begin{figure}[ht]
    \centering
    \begin{minipage}[t]{0.31\textwidth}
        \centering
        \textbf{Classical Bit}
        \vspace{0.2cm}
        
        \begin{tikzpicture}[scale=0.8]
            \draw[thick] (0,0) -- (1,0);
            \draw[thick, fill=blue!20] (0.2,-0.3) rectangle (0.8,0.3);
            \node at (0.5,0) {\textbf{X}};
            \node[left] at (0,0) {0};
            \node[right] at (1,0) {1};
        \end{tikzpicture}
        
        \vspace{0.4cm}
        
        \begin{tikzpicture}[scale=0.6]
            \draw[thick, ->] (0,0) -- (0,1.5) node[above] {1};
            \draw[thick, ->] (0,0) -- (0,-1.5) node[below] {0};
            \fill[blue] (0,1.4) circle (3pt);
            \fill[blue] (0,-1.4) circle (3pt);
        \end{tikzpicture}
        
        \vspace{0.2cm}
        \footnotesize
        \begin{itemize}
            \raggedright
            \item \textbf{Entity:} Discrete Bit $\in \{0, 1\}$.
            \item \textbf{Operation:} Logic Gate (Boolean Switch).
            \item \textbf{Dynamics:} Jumps between isolated states.
        \end{itemize}
    \end{minipage}\hfill
    \begin{minipage}[t]{0.31\textwidth}
        \centering
        \textbf{Unit Circle}
        \vspace{0.2cm}
        
        \begin{tikzpicture}[scale=0.8]
            \draw[thick] (0,0) -- (1.5,0);
            \draw[thick, fill=green!20] (0.4,-0.3) rectangle (1.1,0.3);
            \node at (0.75,0) {$e^{i\theta}$};
            \node[left] at (0,0) {\scriptsize $\phi$};
            \node[right] at (1.5,0) {\scriptsize $\phi+\theta$};
        \end{tikzpicture}
        
        \vspace{0.4cm}
        
        \begin{tikzpicture}[scale=0.6]
            \draw[->, gray] (-1.5,0) -- (1.5,0) node[right, font=\tiny] {$\mathbb{R}$};
            \draw[->, gray] (0,-1.5) -- (0,1.5) node[above, font=\tiny] {$i\mathbb{R}$};
            \draw[thick, green!70!black] (0,0) circle (1.2);
            \draw[thick, ->, green!50!black] (0,0) -- (0.848,0.848) node[above right] {$e^{i\phi}$};
            \fill (0,0) circle (1.5pt);
        \end{tikzpicture}

        \vspace{0.2cm}
        \footnotesize
        \begin{itemize}
            \raggedright
            \item \textbf{Entity:} Continuous Phase $\phi \in (-\pi, \pi]$.
            \item \textbf{Operation:} Phase Shift (Scalar Mult.).
            \item \textbf{Dynamics:} Smooth rotation on a circle.
        \end{itemize}
    \end{minipage}\hfill
    \begin{minipage}[t]{0.31\textwidth}
        \centering
        \textbf{Qubit}
        \vspace{0.2cm}
        
        \begin{tikzpicture}[scale=0.8]
            \draw[thick] (0,0) -- (1.5,0);
            \draw[thick, fill=purple!20] (0.4,-0.3) rectangle (1.1,0.3);
            \node at (0.75,0) {\textbf{H}};
            \node[left] at (0,0) {\scriptsize $|0\rangle$};
            \node[right] at (1.5,0) {\scriptsize $|+\rangle$};
        \end{tikzpicture}
        
        \vspace{0.4cm}

        \begin{tikzpicture}[scale=0.6]
            \draw[dashed, gray] (1.2,0) arc (0:180:1.2 and 0.4);
            \draw[gray] (-1.2,0) arc (180:360:1.2 and 0.4);
            \draw[gray] (0,0) circle (1.2);
            \draw[->, gray] (0,0) -- (0,1.5) node[above, font=\tiny] {$|0\rangle$};
            \draw[->, gray] (0,0) -- (0,-1.5) node[below, font=\tiny] {$|1\rangle$};
            \draw[thick, ->, purple] (0,0) -- (0.848,0.5) node[right, font=\tiny] {$|\psi\rangle$};
            \fill (0,0) circle (1pt);
        \end{tikzpicture}

        \vspace{0.2cm}
        \footnotesize
        \begin{itemize}
            \raggedright
            \item \textbf{Entity:} Superposition $\alpha|0\rangle + \beta|1\rangle$.
            \item \textbf{Operation:} Quantum Gate (Unitary Matrix).
            \item \textbf{Dynamics:} Complex rotation in Hilbert Space.
        \end{itemize}
    \end{minipage}
    \caption{The three paradigms of computation. PhasorFlow introduces the Unit Circle paradigm as a continuous, deterministic bridge between discrete classical bits and complex, non-deterministic quantum qubits.}
    \label{fig:three_paradigms}
\end{figure}

The unit circle $S^1$, corresponding to the unitary group $\unitary(1)$, occupies the middle ground of this ladder.
It is the simplest continuous group that supports interference, periodicity, and Fourier analysis---three properties that are foundational to both signal processing and quantum mechanics \cite{hirose2012complex, tygert2016importance}.
A computational element on $S^1$, which we term a \emph{phasor}, is a complex number of unit modulus: $z = e^{i\phi}$, where $\phi \in (-\pi, \pi]$ is a continuous phase angle.
Unlike qubits, whose state vectors reside in a linear Hilbert space allowing for true quantum superposition, phasor states begin strictly constrained to the $N$-Torus ($\mathbb{T}^N$), a compact, non-linear manifold. 
Because the $N$-Torus is not closed under addition, linear interference (mixing) naturally shifts the state off the manifold into the broader $\mathbb{C}^N$ complex space. Unlike rigid digital logic, this transient departure from unit magnitude allows phasor networks to naturally scale continuous wave interference dynamically across layers.
This comparison is summarized in \Cref{tab:three_paradigms}.

\begin{table}[ht]
    \centering
    \caption{Summary matrix of the three computational paradigms.}
    \label{tab:three_paradigms}
    \begin{tabular}{@{}lccc@{}}
        \toprule
        \textbf{Feature} & \textbf{Classical Bit} & \textbf{Unit Circle} & \textbf{Qubit} \\ 
        \midrule
        \textbf{State Space} & $2^N$ Discrete Points & Continuous Angles & Hilbert Space \\ 
        \textbf{Parameters} & 0 (Rigid Logic) & $N$ (Linear Scaling) & $2^{N+1}-2$ (Exponential) \\ 
        \textbf{Manifold} & 0D Hypercube & $N$-Torus ($\mathbb{T}^N$) & Complex Projective $\mathbb{CP}^{2^N-1}$ \\ 
        \textbf{Gate Type} & Logic (AND/OR) & Rotation Matrices & Unitary Matrices \\ 
        \textbf{Execution} & Deterministic & Deterministic (Phase) & Probabilistic (Superposition) \\ 
        \textbf{Connection} & Wire Connectors & Phase Coupling / Mixing & Entanglement \\ 
        \bottomrule
    \end{tabular}
\end{table}

In this paper, we present \textbf{PhasorFlow}, a Python library that provides a complete framework for \emph{unit circle based computing}.
PhasorFlow draws design inspiration from quantum computing frameworks such as Qiskit \cite{qiskit2024}, adopting a circuit-based programming model in which a user defines a circuit of $N$ phasor threads (unit circles) and applies a sequence of $M$ gate operations.
However, unlike quantum simulators, PhasorFlow circuits are evaluated analytically through direct matrix multiplication, yielding deterministic outputs without sampling noise.

The development of PhasorFlow is fundamentally motivated by the prevalence of complex spatio-temporal dynamics across diverse real-world domains. Many natural and artificial systems are inherently oscillatory or cyclical, rendering traditional Euclidean representations less effective. Continuous dynamical systems often present complex periodic boundary interactions that are difficult to isolate with standard dense linear layers. In neuroscience, Brain-Computer Interface (BCI) data consists of continuous, high-density neural signals that are fundamentally spatio-temporal and oscillatory in nature, naturally demanding phase-based analysis \cite{buzsaki2006rhythms}. By elevating computation natively to the unit circle, PhasorFlow provides a mathematical architecture optimally aligned for encoding, mixing, and extracting phase-based interaction dynamics inherent in these complex systems.

The contributions of this work are threefold:

\begin{enumerate}
    \item \textbf{Phasor Circuit Formalism.}
    We define a complete gate algebra for unit circle computation comprising an expanded library of 22 gates categorized into Standard Unitary, Non-Linear, Neuromorphic, and Encoding operations.
    We prove that the linear gate set forms a unitary group action on the $N$-torus $\mathbb{T}^N = (S^1)^N$, and demonstrate the capability of PhasorFlow to perform Neuromorphic computing (such as neural binding and oscillatory associative memory) natively.

    \item \textbf{Variational Phasor Circuits (VPC).}
    We introduce trainable phasor circuits for machine learning, analogous to Variational Quantum Circuits \cite{benedetti2019parameterized, cerezo2021variational} but operating entirely on classical hardware; full mathematical treatment, capacity analysis, and real motor-imagery EEG validation are presented in a companion manuscript \cite{sigdel2026vpc}.

    \item \textbf{Phasor Transformer and Large Phasor Model (LPM).}
    Inspired by the FNet architecture \cite{lee2021fnet}, we construct a Phasor Transformer that replaces $\mathcal{O}(n^2)$ self-attention with a parameter-free DFT-based token mixing layer, and define the LPM as a deep stack of such blocks; full architecture specification, depth scaling, and time-series benchmarks are presented in a companion manuscript \cite{sigdel2026lpm}.
\end{enumerate}

The remainder of this paper is organized as follows.
\Cref{sec:theory} establishes the mathematical foundations of $\unitary(1)$ phasor computing.
\Cref{sec:methods} details the three methodological contributions: phasor circuits, VPCs (overview), and Phasor Transformers (overview).
\Cref{sec:implementation} describes the PhasorFlow software architecture.
\Cref{sec:applications} presents framework-unique applications: financial volatility detection, algorithmic logic, period finding, and neuromorphic computing.
\Cref{sec:results} reports experimental results for the framework applications.
\Cref{sec:discussion} and \Cref{sec:conclusion} provide discussion, future directions, and concluding remarks.


\section{Theory}
\label{sec:theory}

This section establishes the mathematical framework for unit circle computing.
We first define phasor states on the $N$-torus, then formalize the unitary operators that act on those states. Throughout this section, $\phi$ denotes phasor state angles and $\theta$ denotes unitary-operator parameters.

\subsection{State Space}
\label{subsec:state_space}

\begin{definition}[Phasor Circuit State]
For $N$ computational threads, a phasor circuit state is a vector
\begin{equation}
\boldsymbol{z}=(z_1,\dots,z_N)^\top\in\complex^N,\qquad z_k=e^{i\phi_k},\ \phi_k\in(-\pi,\pi],
\end{equation}
with admissible encoded states constrained to $\mathbb{T}^N=(S^1)^N$ by $|z_k|=1$ for all $k$. We adopt the principal-branch convention $\phi_k=\arg(z_k)\in(-\pi,\pi]$ throughout, consistent with the companion VPC \cite{sigdel2026vpc} and LPM \cite{sigdel2026lpm} papers and with the $\mathrm{atan2}$-based phase extraction used in the implementation.
\end{definition}

A phasor circuit with $N$ independent unit circles (which we term \emph{threads}) has its state residing on the $N$-torus:
\begin{equation}
    \mathbb{T}^N = \underbrace{S^1 \times S^1 \times \cdots \times S^1}_{N} = (S^1)^N.
    \label{eq:n_torus}
\end{equation}

The state of the system is represented as a complex vector:
\begin{equation}
    \boldsymbol{z} = \begin{pmatrix} z_1 \\ z_2 \\ \vdots \\ z_N \end{pmatrix} = \begin{pmatrix} e^{i\phi_1} \\ e^{i\phi_2} \\ \vdots \\ e^{i\phi_N} \end{pmatrix} \in \complex^N, \quad |z_k| = 1 \;\forall\; k.
    \label{eq:state_vector}
\end{equation}

The initial state of the circuit is defined as all phasors at zero phase:
\begin{equation}
    \boldsymbol{z}_0 = \begin{pmatrix} 1 \\ 1 \\ \vdots \\ 1 \end{pmatrix},
    \label{eq:initial_state}
\end{equation}corresponding to $\phi_k = 0$ for all threads $k = 1, \ldots, N$.

It is crucial to distinguish this $\mathbb{T}^N$ geometry from a standard linear vector space (such as the $\mathbb{R}^N$ hidden states of classical neural networks or the $\mathbb{C}^{2^N}$ Hilbert space of quantum mechanics). The $N$-Torus is a compact, non-linear manifold. While the state is represented mathematically as a complex vector, the strict constraint that $|z_k| = 1$ implies that the linear addition of two valid states $\boldsymbol{z}_A + \boldsymbol{z}_B$ generally does not produce a valid state on the torus. The loss of the linear superposition principle is traded for absolute geometric stability---the state cannot explode to infinity.

\begin{proposition}[Torus-Preserving Diagonal Action]
Let $D=\mathrm{diag}(e^{i\theta_1},\dots,e^{i\theta_N})\in U(1)^N$ be a diagonal phase-shift operator with gate parameters $\theta_k\in\mathbb{R}$ (the map $\theta\mapsto e^{i\theta}$ is $2\pi$-periodic), and let $\boldsymbol{z}\in\mathbb{T}^N$. Then $D\boldsymbol{z}\in\mathbb{T}^N$.
\end{proposition}
\begin{proof}
Each coordinate transforms as $(D\boldsymbol{z})_k=e^{i\theta_k}z_k$ and therefore $|(D\boldsymbol{z})_k|=|e^{i\theta_k}|\,|z_k|=1$.
\end{proof}

\begin{figure}[ht]
\centering
\begin{tikzpicture}[scale=0.9, transform shape]
    \begin{scope}[shift={(0,0)}]
        \draw[thick, blue!80, fill=blue!10] (0,0) circle (0.8);
        \draw[->, thick, red] (0,0) -- (45:0.8) node[above right, text=black] {$\phi_1$};
        \fill[black] (0,0) circle (1.5pt);
        \node[yshift=-1.2cm] at (0,0) {1 Unit ($S^1$)};
    \end{scope}

    \begin{scope}[shift={(4,0)}]
        \draw[thick, blue!80, fill=blue!10] (0,0) ellipse (1.5 and 0.8);
        \draw[thick, blue!80] (-0.6,0.1) arc (140:40:0.8);
        \draw[thick, blue!80] (-0.5,0.25) arc (-150:-30:0.6);
        
        \draw[thick, red] (1.05, 0) ellipse (0.45 and 0.55);
        \node[text=red] at (1.65, -0.4) {$\phi_1$};
        
        \draw[thick, green!60!black, dashed] (0,0) ellipse (1.2 and 0.4);
        \draw[thick, green!60!black] (-1.2,0) arc (180:360:1.2 and 0.4);
        \node[text=green!60!black] at (0, -0.6) {$\phi_2$};
        
        \fill[black] (1.05,-0.2) circle (1.5pt);
        
        \node[yshift=-1.2cm] at (0,0) {2 Units ($T^2$)};
    \end{scope}

    \begin{scope}[shift={(9,0)}]
        \node[yshift=-1.2cm] at (0.5,0) {3 Units ($T^3$)};
        \draw[thick, gray!60, dashed] (0,0) -- (1.5,0) -- (1.5,1.5) -- (0,1.5) -- cycle;
        \draw[thick, gray!60, dashed] (0.5,0.5) -- (2.0,0.5) -- (2.0,2.0) -- (0.5,2.0) -- cycle;
        \draw[thick, gray!60, dashed] (0,0) -- (0.5,0.5); \draw[thick, gray!60, dashed] (1.5,0) -- (2.0,0.5);
        \draw[thick, gray!60, dashed] (1.5,1.5) -- (2.0,2.0); \draw[thick, gray!60, dashed] (0,1.5) -- (0.5,2.0);
        
        \draw[->, thick, red] (0,0) -- (1.5,0) node[midway, below, text=black] {$\phi_1$};
        \draw[->, thick, green!60!black] (0,0) -- (0,1.5) node[midway, left, text=black, xshift=-0.05cm] {$\phi_2$};
        \draw[->, thick, orange] (0,0) -- (0.5,0.5) node[right, text=black] {$\phi_3$};
        
        \node[text=red, fill=white, inner sep=0.5pt] at (0.75,1.5) {\footnotesize $\approx$};
        \node[text=red, fill=white, inner sep=0.5pt] at (0.75,0) {\footnotesize $\approx$};
        \node[text=green!60!black, fill=white, inner sep=0.5pt] at (1.5,0.75) {\footnotesize $\approx$};
        \node[text=green!60!black, fill=white, inner sep=0.5pt] at (0,0.75) {\footnotesize $\approx$};
    \end{scope}
\end{tikzpicture}
\caption{The geometric evolution of the internal phase manifold strictly depends connecting dimension lines on boundaries. An $N$-thread state resolves mathematically onto the periodic coordinate map of an $N$-Torus ($T^N$).}
\label{fig:torus_manifold}
\end{figure}
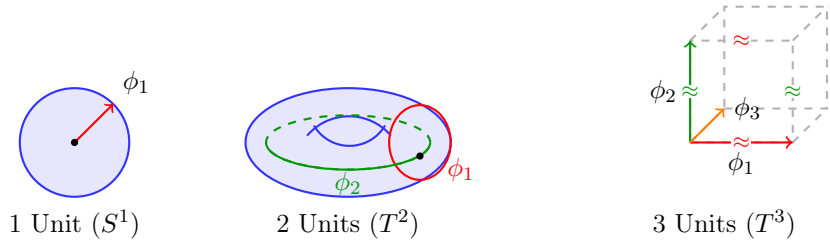

\subsection{Unitary Operations}
\label{subsec:unitary_ops}

Unitary operators act on phasor state vectors in the ambient vector space $\mathbb{C}^N$. A linear operator $U$ is unitary when
\begin{equation}
    U U^\dagger = U^\dagger U = I, \quad \text{where } U^\dagger = \bar{U}^T.
    \label{eq:unitarity}
\end{equation}

The three operator structures used by PhasorFlow are:

\paragraph{$\mathrm{U}(1)$: single-thread phase rotation.}
A one-thread unitary has the scalar form
\begin{equation}
u(\theta)=e^{i\theta}\in \mathrm{U}(1), \qquad \theta\in\mathbb{R}.
\label{eq:u1_scalar}
\end{equation}
Embedded into an $N$-thread state, this becomes a diagonal operator acting on one coordinate:
\begin{equation}
S_k(\theta)=\mathrm{diag}(1,\ldots,e^{i\theta},\ldots,1)\in \mathrm{U}(N).
\label{eq:u1_embedded}
\end{equation}

\paragraph{$\mathrm{U}(2)$: two-thread mixing.}
A two-thread mixing operator acts on a 2D subspace as
\begin{equation}
M(\theta)=\begin{pmatrix} a(\theta) & b(\theta) \\ c(\theta) & d(\theta) \end{pmatrix}\in \mathrm{U}(2),
\label{eq:u2_generic}
\end{equation}
with orthonormal columns/rows implied by \Cref{eq:unitarity}. When $\det M(\theta)=1$, the operator belongs to $\mathrm{SU}(2)$.

\paragraph{$\mathrm{U}(N)$: global mixing.}
Global couplers act on all threads simultaneously:
\begin{equation}
W(\theta)\in \mathrm{U}(N),
\label{eq:un_generic}
\end{equation}
with the DFT matrix as a canonical parameter-free example in this class.

Hence, Shift, Mix, and DFT are interpreted at the level of group structure as
\begin{equation}
\text{Shift: } \mathrm{U}(1), \qquad
\text{Mix: } \mathrm{U}(2), \qquad
\text{DFT/global mixing: } \mathrm{U}(N).
\label{eq:gate_group_membership}
\end{equation}
If $\det(U)=1$, the corresponding operation is special unitary ($\mathrm{SU}(n)$).

Under a unitary transformation, the state evolves as:
\begin{equation}
    \boldsymbol{z}' = U \boldsymbol{z}.
    \label{eq:state_evolution}
\end{equation}

Since unitary matrices preserve the $\ell^2$-norm of complex vectors, $\|\boldsymbol{z}'\| = \|U\boldsymbol{z}\| = \|\boldsymbol{z}\|$, the total energy of the system is conserved.
For a circuit with $M$ sequential gates $U_1, U_2, \ldots, U_M$, the final state is:
\begin{equation}
    \boldsymbol{z}_{\text{final}} = U_M \cdots U_2 \cdot U_1 \cdot \boldsymbol{z}_0 = \left(\prod_{m=1}^{M} U_m \right) \boldsymbol{z}_0.
    \label{eq:circuit_evolution}
\end{equation}

The product of unitary matrices is itself unitary, so the entire linear circuit represents a single composite isometric operation.

However, a critical geometric shift arises during execution. While a unitary mixing operation (such as $\mathrm{U}(2)$ Mix or global $\mathrm{U}(N)$ mixing) preserves the global $\ell^2$-norm $\|\boldsymbol{z}'\| = \|\boldsymbol{z}\|$, individual thread magnitudes $|z_k'|$ generally drift away from $1$ due to interference. Thus, the state departs from the $N$-torus manifold into the ambient vector space $\mathbb{C}^N$.

This drift is intentionally allowed during linear propagation for representational expressivity. A pull-back mechanism is then introduced in the Methods section (non-linear gates) to re-project states back toward $\mathbb{T}^N$ before the next cycle when topology control is required.

\begin{theorem}[Unitary Energy Conservation with Coordinate Drift]
Let $U\in U(N)$ and $\boldsymbol{z}\in\mathbb{T}^N$. Then
\begin{equation}
\|U\boldsymbol{z}\|_2=\|\boldsymbol{z}\|_2=\sqrt{N}.
\end{equation}
If $U$ is non-diagonal, there exists $\boldsymbol{z}\in\mathbb{T}^N$ such that $U\boldsymbol{z}\notin\mathbb{T}^N$.
\end{theorem}
\begin{proof}
Unitarity gives norm preservation directly. For non-diagonal $U$, some output coordinate is a non-trivial linear combination of at least two unit-modulus entries, so coordinatewise modulus is generally not preserved under interference.
\end{proof}

\begin{corollary}[Ambient-Space Propagation]
Intermediate states of VPC and Phasor Transformer layers are naturally propagated in $\mathbb{C}^N$ even when encoded inputs lie on $\mathbb{T}^N$.
\end{corollary}

\begin{remark}[State Space vs.\ Operator Space]
It is essential to distinguish two distinct mathematical objects throughout this paper.
\begin{enumerate}
  \item \textbf{Phasor state vectors} $\boldsymbol{z}\in\mathbb{C}^N$ are \emph{elements of a vector space}.  When all coordinates satisfy $|z_k|=1$, the state lies on the sub-manifold $\mathbb{T}^N\subset\mathbb{C}^N$; after unitary mixing it may leave $\mathbb{T}^N$ while remaining in $\mathbb{C}^N$.  The phase angle of coordinate $k$ is denoted $\phi_k=\arg(z_k)\in(-\pi,\pi]$.
    \item \textbf{Gate operators} $G\in\mathrm{U}(N)$ are \emph{group elements acting on $\mathbb{C}^N$}.  Their internal parameters are denoted $\theta$.  Shift operations are $\mathrm{U}(1)$ on individual threads (or $\mathrm{U}(1)^N$ diagonally), Mix operations are $\mathrm{U}(2)$ on thread pairs, and DFT operations are global $\mathrm{U}(N)$ maps.  When an operator has determinant one, it is in the corresponding special unitary group $\mathrm{SU}(n)$.
\end{enumerate}
In summary: $\phi$ (Greek phi) always denotes a \emph{phasor angle} (state); $\theta$ (Greek theta) always denotes a \emph{gate rotation parameter} (operator).
\end{remark}

\subsection{Phase Coherence and Interference}
\label{subsec:coherence}

A key observable in phasor systems is the \emph{phase coherence}, which measures the degree of alignment among the $N$ phasors:
\begin{equation}
    \mathcal{C} = \frac{1}{N} \left| \sum_{k=1}^{N} z_k \right| = \frac{1}{N} \left| \sum_{k=1}^{N} e^{i\phi_k} \right|.
    \label{eq:coherence}
\end{equation}

When all phasors are perfectly aligned ($\phi_k = \phi$ for all $k$), the coherence reaches its maximum $\mathcal{C} = 1$.
When phases are uniformly distributed around the circle, destructive interference drives $\mathcal{C} \to 0$.
This observable provides a natural, parameter-free measure of the internal structure of a phasor state, and will be used as a feature detector in our applications.

\subsection{Manifold Comparison}
\label{subsec:comparison}

\Cref{tab:manifold_comparison} summarizes the key differences between the computational manifolds on the Geometric Ladder.

\begin{table}[ht]
    \centering
    \caption{Comparison of computational manifolds.}
    \label{tab:manifold_comparison}
    \begin{tabular}{@{}lccc@{}}
        \toprule
        \textbf{Property} & \textbf{Bit} & \textbf{Phasor ($S^1$)} & \textbf{Qubit ($\mathbb{CP}^1$)} \\
        \midrule
        Manifold dimension      & 0           & 1                 & 2 \\
        State parameters        & 1 (binary)  & 1 ($\phi$)        & 2 ($\theta, \phi$) \\
        Deterministic           & Yes         & Yes               & No \\
        Interference            & No          & Yes               & Yes \\
        Superposition           & No          & No                & Yes \\
        Group structure         & $\mathbb{Z}_2$ & $\unitary(1)$  & $\mathrm{SU}(2)$ \\
        Hardware requirement    & Classical   & Classical         & Quantum \\
        \bottomrule
    \end{tabular}
\end{table}

The phasor model inherits the classical wave interference property (enabling constructive and destructive combination) while remaining fully deterministic and executable on standard architecture without requiring quantum superposition or non-local quantum entanglement.
This positions $S^1$ computing as an intermediate paradigm---more expressive than bit-based computing for oscillatory dynamics, yet fundamentally classical and strictly practical for immediate application.


\section{Method}
\label{sec:methods}

This section presents the three principal methodological contributions of PhasorFlow: the Phasor Circuit model (\Cref{subsec:phasor_circuits}), the Variational Phasor Circuit for machine learning (\Cref{subsec:vpc}), and the Phasor Transformer architecture for sequence modeling (\Cref{subsec:phasor_transformer}).

\subsection{Phasor Circuits}
\label{subsec:phasor_circuits}

A \emph{Phasor Circuit} $\mathcal{P}(N, M)$ is defined by $N$ unit circle threads and an ordered sequence of $M$ gate operations.
The circuit maps an initial state $\boldsymbol{z}_0 \in \mathbb{T}^N$ to a final state $\boldsymbol{z}_f$ through sequential application of gate matrices.
The PhasorFlow library organizes 22 core gate primitives into four operational categories: Standard Unitary, Non-Linear, Neuromorphic, and Encoding gates.

\begin{figure}[ht]
\centering
\begin{tikzpicture}[scale=1.2]
    \foreach \i in {1,...,5} {
        \draw[thick] (0, 6-\i) node[left] {$\phi_\i$} -- (7.5, 6-\i);
    }
    
    \draw[thick, fill=green!10] (1.0, 4.7) rectangle (1.8, 5.3); \node at (1.4, 5.0) {$S$};
    \draw[thick, fill=green!10] (1.0, 1.7) rectangle (1.8, 2.3); \node at (1.4, 2.0) {$S$};
    
    \draw[thick, fill=orange!10] (2.5, 3.8) rectangle (3.3, 5.2); \node at (2.9, 4.5) {$M$};
    \draw[thick, fill=orange!10] (2.5, 1.8) rectangle (3.3, 3.2); \node at (2.9, 2.5) {$M$};
    
    \draw[thick, fill=green!10] (4.0, 3.7) rectangle (4.8, 4.3); \node at (4.4, 4.0) {$S$};
    
    \draw[thick, fill=orange!10] (5.5, 2.8) rectangle (6.3, 4.2); \node at (5.9, 3.5) {$M$};
    
    \foreach \i in {1,...,5} {
        \node[right] at (7.5, 6-\i) {$\phi_\i'$};
    }
\end{tikzpicture}
\caption{An example of an $N=5$ continuous Phasor Circuit constructed from parameterized Shift ($S$) gates and fixed entangling Mix ($M$) gates, visually analogous to a parameterized quantum circuit cascade.}
\label{fig:n5_phasor_circuit}
\end{figure}
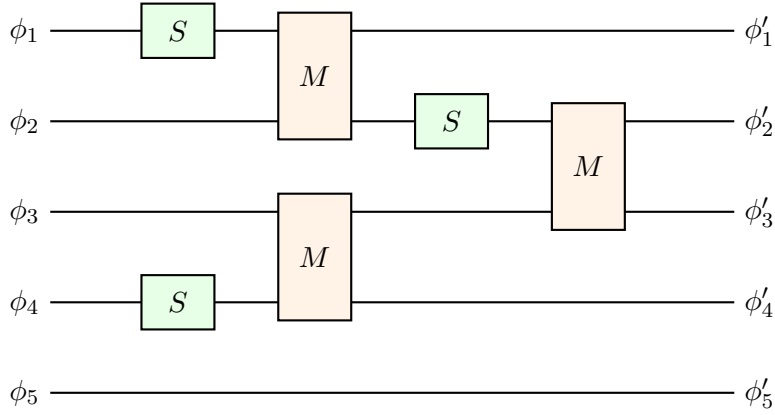
\subsubsection{Shift Gate $S(\theta)$}

The Shift gate applies a phase rotation of trainable angle $\theta \in \mathbb{R}$ (defined modulo $2\pi$) to a single thread $k$.
Acting on the full $N$-dimensional state, it is represented as an $N \times N$ diagonal matrix with $e^{i\theta}$ at position $(k, k)$ and ones elsewhere:
\begin{equation}
    S_k(\theta) = \text{diag}(1, \ldots, \underbrace{e^{i\theta}}_{k\text{-th}}, \ldots, 1).
    \label{eq:shift_gate}
\end{equation}

Since $|e^{i\theta}| = 1$, the diagonal matrix $S_k(\theta)$ is trivially unitary: $S_k(\theta) S_k(\theta)^\dagger = I$.
Applied to a single phasor state $e^{i\phi}$, the $1 \times 1$ matrix form is simply:
\begin{equation}
    S(\theta) = \begin{pmatrix} e^{i\theta} \end{pmatrix},
    \label{eq:shift_1x1}
\end{equation}
producing $S(\theta)\,e^{i\phi} = e^{i(\phi+\theta)}$.
In variational circuits, $\theta$ serves as a trainable phase weight, analogous to rotation parameters in parameterized quantum circuits \cite{schuld2020circuit}.

\textbf{Group membership.}  $S_k(\theta)$ acts as a $\mathrm{U}(1)$ operation on the target thread; embedded in the full $N$-dimensional operator algebra it becomes a diagonal element of $\mathrm{U}(1)^N \subset \mathrm{U}(N)$ with determinant $\det S_k(\theta)=e^{i\theta}$. It is in $\mathrm{SU}(N)$ only when $\theta=0$ (identity). The $N$-fold product of independent Shift gates $\prod_{k=1}^N S_k(\theta_k)$ generates the maximal torus $\mathrm{U}(1)^N$ of $\mathrm{U}(N)$.

\begin{figure}[ht]
\centering
\begin{tikzpicture}[scale=1.2]
    \draw[thick] (0,0) -- (6.0,0);
    \node[left] at (0,0) {$\phi_{\mathrm{in}}$};
    \node[right] at (6.0,0) {$\phi_{\mathrm{out}} = \phi_{\mathrm{in}} + \theta$};
    \draw[thick, fill=green!10] (0.5,-0.4) rectangle (5.8,0.4);
    \node at (3.15,0) {\textbf{Shift} $S(\theta)$: $e^{i\phi} \mapsto e^{i(\phi+\theta)}$};
\end{tikzpicture}
\caption{Circuit representation of the Shift gate acting on a single computation thread.}
\label{fig:shift_gate}
\end{figure}
\subsubsection{Mix Gate $M_{jk}$}

The Mix gate creates interference between two threads $j$ and $k$.
It is defined as a $2 \times 2$ unitary matrix that acts as a 50/50 beam splitter:
\begin{equation}
    M_{jk} = \frac{1}{\sqrt{2}} \begin{pmatrix} 1 & i \\ i & 1 \end{pmatrix}.
    \label{eq:mix_gate}
\end{equation}

Verification of unitarity:
\begin{equation}
    M_{jk} M_{jk}^\dagger = \frac{1}{2} \begin{pmatrix} 1 & i \\ i & 1 \end{pmatrix} \begin{pmatrix} 1 & -i \\ -i & 1 \end{pmatrix} = \begin{pmatrix} 1 & 0 \\ 0 & 1 \end{pmatrix} = I.
    \label{eq:mix_unitary_proof}
\end{equation}

The Mix gate transforms a pair of phasors $(z_j, z_k)$ as:
\begin{equation}
    \begin{pmatrix} z_j' \\ z_k' \end{pmatrix} = \frac{1}{\sqrt{2}} \begin{pmatrix} z_j + i z_k \\ i z_j + z_k \end{pmatrix}.
    \label{eq:mix_action}
\end{equation}
This operation introduces phase-dependent coupling: the output phases depend on the \emph{relative} phase difference $\phi_j - \phi_k$ between the two input threads.
In the context of neural networks, the Mix gate functions as a fixed (non-parameterized) coupling layer that prevents trivial factorization of the circuit into independent single-thread operations.

\textbf{Group membership.} $M_{jk}$ is an element of $\mathrm{SU}(2)$ on the two-thread subspace: one can verify $\det(M_{jk}) = \tfrac{1}{2}(1\cdot1 - i\cdot i) = 1$, so $M_{jk}\in\mathrm{SU}(2)\subset\mathrm{U}(2)$. When embedded into the $N$-dimensional space it acts as the identity on all other threads, yielding $M_{jk}\in\mathrm{U}(N)$.

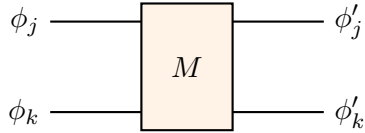
\begin{figure}[ht]
\centering
\begin{tikzpicture}[scale=1.2]
    \draw[thick] (0,1) node[left] {$\phi_j$} -- (3,1) node[right] {$\phi_j'$};
    \draw[thick] (0,0) node[left] {$\phi_k$} -- (3,0) node[right] {$\phi_k'$};
    \draw[thick, fill=orange!10] (1,-0.2) rectangle (2,1.2);
    \node at (1.5,0.5) {$M$};
\end{tikzpicture}
\caption{Circuit representation of the Mix gate, entangling two adjacent continuous phase threads.}
\label{fig:mix_gate}
\end{figure}

\subsubsection{Discrete Fourier Transform (DFT) Gate}

The DFT gate applies a global $N \times N$ unitary transformation across all threads simultaneously, mixing all phases through the discrete Fourier basis:
\begin{equation}
    F_N = \frac{1}{\sqrt{N}} \begin{pmatrix}
        1 & 1 & 1 & \cdots & 1 \\
        1 & \omega & \omega^2 & \cdots & \omega^{N-1} \\
        1 & \omega^2 & \omega^4 & \cdots & \omega^{2(N-1)} \\
        \vdots & \vdots & \vdots & \ddots & \vdots \\
        1 & \omega^{N-1} & \omega^{2(N-1)} & \cdots & \omega^{(N-1)^2}
    \end{pmatrix},
    \label{eq:dft_gate}
\end{equation}
where $\omega = e^{-2\pi i / N}$ is the $N$-th root of unity.

The DFT matrix is unitary by construction: each row (and column) forms an orthonormal set under the standard inner product on $\complex^N$ \cite{cooley1965algorithm}.
The action on the state vector transforms from the ``spatial'' domain to the ``frequency'' domain:
\begin{equation}
    z_k' = \frac{1}{\sqrt{N}} \sum_{n=0}^{N-1} z_n \, \omega^{kn}, \quad k = 0, 1, \ldots, N-1.
    \label{eq:dft_action}
\end{equation}

Unlike the Shift and Mix gates, the DFT gate couples \emph{all} $N$ threads simultaneously, creating global interference patterns.
This makes it the most powerful mixing operation in PhasorFlow, and it plays a central role in the Phasor Transformer architecture (\Cref{subsec:phasor_transformer}).

\textbf{Group membership.} $F_N\in\mathrm{U}(N)$ is a global isometry of the state space $\mathbb{C}^N$; each row/column is an orthonormal Fourier basis vector. Whether $F_N\in\mathrm{SU}(N)$ depends on $N$ (the determinant is a root of unity whose exact value depends on $N$ modulo 4), but in all cases $|\det F_N|=1$.

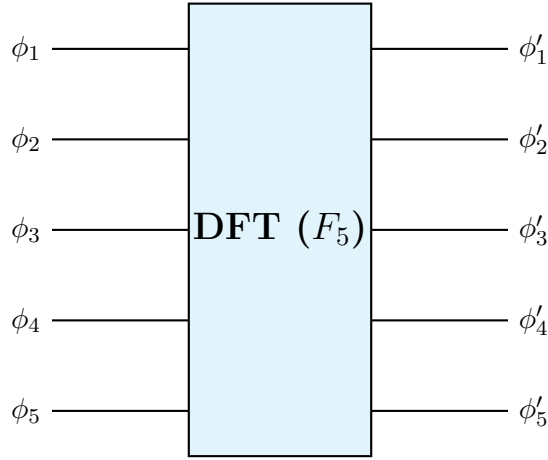
\begin{figure}[ht]
\centering
\begin{tikzpicture}[scale=1.2]
    \foreach \i in {1,...,5} {
        \draw[thick] (0, 6-\i) node[left] {$\phi_\i$} -- (5, 6-\i) node[right] {$\phi_\i'$};
    }
    
    \draw[thick, fill=cyan!10] (0.8, 0.5) rectangle (4.2, 5.5);
    \node[align=center] at (2.5, 3.0) {\textbf{DFT}\\[0.1cm]\textbf{($F_5$)}};
\end{tikzpicture}
\caption{Circuit representation of the $N=5$ Discrete Fourier Transform (DFT) gate. Unlike strictly local or pairwise operations, the DFT acts intrinsically as an all-to-all global unitary operator, extracting Fourier-basis classical wave interference phases across the entire thread registry simultaneously.}
\label{fig:dft_gate}
\end{figure}

\subsubsection{Invert Gate}

The Invert gate is a special case of the Shift gate with $\theta = \pi$:
\begin{equation}
    I_k = S_k(\pi) = \text{diag}(1, \ldots, \underbrace{-1}_{k\text{-th}}, \ldots, 1),
    \label{eq:invert_gate}
\end{equation}
which maps $z_k \mapsto -z_k = e^{i(\phi_k + \pi)}$, reflecting the phasor across the origin.

\subsubsection{Additional Standard Unitary Gates}

Beyond the core mixing and phase shift operations, PhasorFlow includes structural wire routing and aggregation gates:

\paragraph{Permute Gate.}
The Permute gate reorders thread indices natively without breaking continuous wave topologies. Given a permutation vector $\pi = (\pi_1, \ldots, \pi_N)$, it maps the system state as:
\begin{equation}
    P_\pi(\boldsymbol{z}) = (z_{\pi_1}, z_{\pi_2}, \ldots, z_{\pi_N}).
    \label{eq:permute_gate}
\end{equation}

\paragraph{Reverse Gate.}
The Reverse gate executes a time-reversal operator by globally conjugating the complex state vector across all $N$ threads:
\begin{equation}
    R(\boldsymbol{z}) = \boldsymbol{z}^* = (z_1^*, z_2^*, \ldots, z_N^*).
    \label{eq:reverse_gate}
\end{equation}

\paragraph{Accumulate Gate.}
The Accumulate gate performs a local cumulative complex summation sweeping sequentially across adjacent threads:
\begin{equation}
    A(\boldsymbol{z})_k = \sum_{j=1}^{k} z_j, \quad k = 1, \ldots, N.
    \label{eq:accumulate_gate}
\end{equation}
This coherent phase summation powers wave-front dynamics, seamlessly generating recursive structures like the Fibonacci sequence strictly through continuous wave interference.

\paragraph{GridPropagate Gate.}
For 2D dynamic programming, the GridPropagate gate simulates wavefront propagation across a lattice topology, where the state of each node $(r, c)$ accumulates its top and left neighbors:
\begin{equation}
    G(\boldsymbol{Z})_{r,c} = \boldsymbol{Z}_{r-1,c} + \boldsymbol{Z}_{r,c-1},
    \label{eq:gridpropagate_gate}
\end{equation}
enabling native unit-circle evaluation of localized connectivity graphs.

\subsubsection{Gate Library Summary}

PhasorFlow ships with a comprehensive library of 22 primitive gates that organically compose into fully functional algorithmic and machine learning pipelines. \Cref{tab:all_gates} summarizes the complete operational toolkit.

\begin{table}[ht]
    \centering
    \caption{Summary of all 22 primitive gates available in the PhasorFlow computing framework, categorized by operation type.}
    \label{tab:all_gates}
    \resizebox{\textwidth}{!}{
    \begin{tabular}{@{}lll@{}}
        \toprule
        \textbf{Category} & \textbf{Gate Name} & \textbf{Operation / Function Description} \\
        \midrule
        \multirow{8}{*}{\textbf{Standard Unitary}} 
        & Shift & Phase rotation proportional to input value ($z \mapsto z \cdot e^{i\theta}$). \\
        & Invert & Phase flip by $\pi$ radians ($z \mapsto -z$). \\
        & Mix & Two-thread interference (beam splitter). \\
        & DFT & Global sequence token mixing via Discrete Fourier Transform. \\
        & Permute & Reordering of computing thread state indices. \\
        & Reverse & Time-reversal via global complex conjugation ($z \mapsto z^*$). \\
        & Accumulate & Cumulative complex wave summation ($z_{n+1} = z_{n+1} + z_n$). \\
        & GridPropagate & Wavefront propagation accumulation across a 2D lattice. \\
        \midrule
        \multirow{6}{*}{\textbf{Non-Linear}} 
        & Threshold & Filters low-magnitude phasors and forces outputs to zero. \\
        & Saturate & Quantizes phase geometry toward discrete binary anchors. \\
        & Normalize & Pulls any generic $\mathbb{C}^N$ state rigidly back to the $\mathbb{T}^N$ unit circle. \\
        & LogCompress & Logarithmic amplitude compression ($\mu$-law analog). \\
        & CrossCorrelate & Evaluates phase coherence between discrete pattern sequences. \\
        & Convolve & Sliding continuous spatial convolution along threads. \\
        \midrule
        \multirow{5}{*}{\textbf{Neuromorphic}} 
        & Kuramoto & Global phase synchronization towards a mean alignment field. \\
        & Hebbian & Associative memory adaptation via nearest-neighbor phase pull. \\
        & Ising & Anti-ferromagnetic coupling driving bi-modal ($0, \pi$) symmetry. \\
        & Synaptic & Continuous drag/coupling between targeted neural oscillators. \\
        & AsymmetricCouple & Non-reciprocal directed phase influence across nodes. \\
        \midrule
        \multirow{2}{*}{\textbf{Encoding}} 
        & EncodePhase & Maps real-valued $r_i$ into the spatial phase domain $(-\pi, \pi]$. \\
        & EncodeAmplitude & Maps scalar magnitudes physically onto the wave norm. \\
        \bottomrule
    \end{tabular}
    }
\end{table}

\subsubsection{Non-Linear Gates}

While the linear gates (Shift, Mix, DFT) execute unitary transformations that conserve global energy, they dynamically shift the individual threads away from the unit magnitude constraint. Moving the state vector off the $N$-Torus manifold into the full $\mathbb{C}^N$ complex space expands the system's ability to natively scale latent interaction magnitudes. For situations requiring explicit topology control or discrete programmatic decisions, PhasorFlow provides optional non-linear gates that act as geometric projections, breaking unitarity to force the state back towards $\mathbb{T}^N$.

\paragraph{Threshold Gate.}
The Threshold gate acts as a non-linear activation function on the phasor amplitude, optionally performing a rigid re-normalization step.
Given a complex value $z$ with magnitude $|z|$ and a threshold parameter $\tau$:
\begin{equation}
    T(\tau)(z) = \begin{cases}
        z / |z| & \text{if } |z| \geq \tau, \\
        0 & \text{if } |z| < \tau.
    \end{cases}
    \label{eq:threshold_gate}
\end{equation}

After Mix or DFT operations, individual thread amplitudes deviate from unity. The Threshold gate provides a discrete decision mechanism: phasors with sufficient geometric stability are rigidly mapped straight back to the $\mathbb{T}^N$ manifold, while weak destructive signals are strictly suppressed. While this hard-constrains intermediate states, empirical models like the deep Variational Phasor Circuit often bypass this thresholding entirely to exploit continuous complex wave topologies.

\paragraph{Saturate Gate.}
The Saturate gate discretizes the continuous phase angle to one of $L$ equally spaced levels:
\begin{equation}
    \text{Sat}(L)(z) = \exp\left(i \cdot \text{round}\left(\frac{\theta}{2\pi/L}\right) \cdot \frac{2\pi}{L}\right), \quad \theta = \arg(z).
    \label{eq:saturate_gate}
\end{equation}

For $L=2$, the Saturate gate snaps phases to either $0$ or $\pi$, effectively quantizing the continuous phasor to a binary representation.
This gate enables error correction in oscillatory memory networks and serves as the physical foundation for Hopfield-type attractor dynamics \cite{hopfield1982, hoppensteadt1999associative, fiete2010high, ramsauer2021modern}.

\paragraph{Additional Non-Linear Operations.}
The \textbf{Normalize Gate} (or PullBack) continuously enforces unit-magnitude topology $|z|=1$ iteratively during complex cascade flows. For dynamic range adjustment, the \textbf{LogCompress Gate} attenuates signal extrema. Finally, the \textbf{CrossCorrelate} and \textbf{Convolve} gates execute complex sequence sliding operations directly in the spatial phase domain.

\subsubsection{Neuromorphic Gates}

PhasorFlow introduces a suite of non-unitary associative tracking gates explicitly designed for brain-inspired computing.
\begin{itemize}
    \item \textbf{Kuramoto Gate}: Implements global phase synchronization across continuous populations \cite{kuramoto1975, winfree2001geometry}.
    \item \textbf{Hebbian Gate}: Modifies outer-product associative phase links $\Delta W_{jk} \propto z_j \bar{z}_k$ to store multi-pattern oscillator memories \cite{hoppensteadt1999associative}.
    \item \textbf{Ising Gate}: A discrete coupling operator that drives threads to strictly bipartite consensus arrays.
    \item \textbf{Synaptic \& AsymmetricCouple Gates}: Enact directed phase momentum transfer between distinct computational reservoirs.
\end{itemize}

\subsubsection{Encoding Gates}

The \textbf{EncodePhase Gate} and \textbf{EncodeAmplitude Gate} comprise the native interface for loading external real-valued data structures geometrically onto the $N$-Torus ($\mathbb{T}^N$) manifold or into full continuous $\mathbb{C}^N$ amplitudes.

\subsubsection{Circuit Execution}

A complete circuit $\mathcal{P}(N, M)$ with gate sequence $G_1, G_2, \ldots, G_M$ is executed by applying each gate matrix sequentially to the state vector:
\begin{equation}
    \boldsymbol{z}_f = G_M \cdots G_2 \cdot G_1 \cdot \boldsymbol{z}_0.
    \label{eq:circuit_exec}
\end{equation}

Each gate $G_m$ acts on either a single thread (Shift, Invert), a pair of threads (Mix), or all threads (DFT).
For single-thread and pair-thread gates, the operation is embedded into the full $N$-dimensional space by acting as the identity on all non-target threads.
The computational cost of circuit execution is $\mathcal{O}(M \cdot N^2)$ in the general case, dominated by the DFT gate applications.

\textbf{Composite group structure.} The sequential product of $M$ linear (unitary) gates yields a total circuit operator
\begin{equation}
  U_{\mathrm{circ}} = G_M \cdots G_1 \in \mathrm{U}(N),
  \label{eq:circuit_unitary}
\end{equation}
because the product of unitary matrices is unitary.  If $\det G_m = 1$ for every gate $m$ (as is the case for Mix gates), then $U_{\mathrm{circ}}\in\mathrm{SU}(N)$; Shift gates with $\theta\neq 0$ contribute a unit-modulus phase factor to $\det U_{\mathrm{circ}}$, keeping it in $\mathrm{U}(N)\setminus\mathrm{SU}(N)$ in general.  Non-linear gates (Threshold, Normalize, Saturate, etc.) break unitarity and are therefore \emph{not} elements of $\mathrm{U}(N)$; they act as projections onto sub-manifolds of $\mathbb{C}^N$.

\subsubsection{Leaky-Integrate-and-Phase (LIP) Layer}

The LIP layer provides a continuous-time dynamical system for $N$ coupled oscillators, inspired by the Kuramoto model \cite{kuramoto1975}:
\begin{equation}
    \frac{d\phi_k}{dt} = -\gamma (\phi_k - \phi_{\text{rest}}) + \sum_{j=1}^{N} W_{kj} \sin(\phi_k - \phi_j) + I_k^{\text{ext}},
    \label{eq:lip_dynamics}
\end{equation}
where $\gamma$ is the leak rate, $\phi_{\text{rest}}$ is the resting phase, $W_{kj}$ are synaptic coupling weights, and $I_k^{\text{ext}}$ is external input.

The LIP layer enables simulation of neural binding phenomena---the process by which distributed neural populations achieve phase synchronization to represent a unified percept.

\subsubsection{Associative Memory (Hopfield-Phase Model)}

The \texttt{PhasorFlowMemory} class implements an oscillatory Hopfield network that stores and retrieves phase patterns via Hebbian learning:
\begin{equation}
    W = \frac{1}{P} \sum_{p=1}^{P} \boldsymbol{z}^{(p)} \left(\boldsymbol{z}^{(p)}\right)^\dagger, \quad W_{kk} = 0,
    \label{eq:hebbian_weights}
\end{equation}
where $\boldsymbol{z}^{(p)} = e^{i \boldsymbol{\phi}^{(p)}}$ are the stored phase patterns and $P$ is the number of patterns.

Pattern recovery from a corrupted input proceeds via iterative phase-locking:
\begin{equation}
    \boldsymbol{z}^{(t+1)} = \frac{\boldsymbol{z}^{(t)} + \delta t \cdot W \boldsymbol{z}^{(t)}}{|\boldsymbol{z}^{(t)} + \delta t \cdot W \boldsymbol{z}^{(t)}|},
    \label{eq:convergence}
\end{equation}
where the division is element-wise and re-normalizes each phasor to the unit circle.
This dynamics converges to the nearest stored attractor, enabling content-addressable memory retrieval \cite{hopfield1982, plate1995holographic}.

\subsection{Variational Phasor Circuit (VPC)}
\label{subsec:vpc}

The Variational Phasor Circuit (VPC) extends the Phasor Circuit model to supervised machine learning by introducing trainable parameters.
The VPC architecture is directly analogous to Variational Quantum Circuits (VQCs) \cite{benedetti2019parameterized, cerezo2021variational} but operates deterministically on classical hardware.

Data are encoded as unit-magnitude complex states $\boldsymbol{z}_\mathrm{in} \in \mathbb{T}^N$, and a single variational layer applies trainable Shift rotations followed by local pairwise Mix coupling:
\begin{equation}
    \mathcal{V}(\boldsymbol{\theta}) = U_\mathrm{local}\!\left(\prod_{k=1}^{N}S_k(\theta_k)\right), \qquad U_\mathrm{local} = \prod_{k=0,2,4,\ldots} M_{k,k+1}.
    \label{eq:vpc_overview}
\end{equation}
Each trainable parameter $\theta_k$ is a phase angle, giving a single-stack VPC precisely $N$ real-valued parameters---linear scaling versus the quadratic parameter count of a fully connected layer of the same width.
Layers can be stacked for added expressivity, and predictions are extracted deterministically from output thread phases without sampling.

\subsection{Phasor Transformer and Large Phasor Model (LPM)}
\label{subsec:phasor_transformer}

The Phasor Transformer adapts the transformer architecture \cite{vaswani2017attention} to unit circle computing by replacing the $\mathcal{O}(n^2)$ self-attention mechanism with the parameter-free DFT gate.
This approach is directly inspired by Google's FNet \cite{lee2021fnet}, which demonstrated that global Fourier mixing achieves strong performance at $\mathcal{O}(n \log n)$ complexity.

A single Phasor Transformer block applies trainable phase-shift projections before and after a DFT token-mixing layer:
\begin{equation}
    \mathcal{B}(\boldsymbol{\theta}) = S(\boldsymbol{\theta}^\mathrm{post})\,F_T\,S(\boldsymbol{\theta}^\mathrm{pre}),
    \label{eq:transformer_block_overview}
\end{equation}
where $S(\cdot)$ is a diagonal phase-rotation and $F_T$ is the DFT token mixer.
Each block carries $2T$ trainable parameters with zero additional cost for the global mixing step.
Stacking $D$ such blocks defines the \textbf{Large Phasor Model (LPM)}.

\begin{figure}[htbp]
\centering
\sffamily
\definecolor{cT3}{HTML}{9C27B0}
\definecolor{cT2}{HTML}{4CAF50}
\definecolor{cT1}{HTML}{03A9F4}
\definecolor{cT0}{HTML}{F44336}
\begin{tikzpicture}[
    scale=0.82, transform shape,
    node distance=1.5cm,
    rail/.style={line width=5pt, line cap=round, opacity=0.8},
    shift_gate/.style={rectangle, draw, fill=white, line width=1.0pt, minimum width=0.6cm, minimum height=0.6cm, font=\small\bfseries, rounded corners=2pt},
    mix_gate/.style={rectangle, draw, fill=gray!10, line width=1.0pt, minimum width=0.4cm, minimum height=1.8cm, font=\small\bfseries, rounded corners=2pt},
    connection/.style={line width=3pt, color=gray!50}
]

    \def\threads{
        3/Token $\phi_3$/cT3,
        2/Token $\phi_2$/cT2,
        1/Token $\phi_1$/cT1,
        0/Token $\phi_0$/cT0%
    }
    
    \foreach \y/\ilabel/\icolor in \threads {
         \pgfmathsetmacro{\ypos}{\y * 1.8}
         \draw[rail, color=\icolor] (0, \ypos) -- (11.0, \ypos);
         \node[anchor=east, color=\icolor, font=\small\bfseries, align=right] at (-0.2, \ypos) {\ilabel};
    }
    
    \node[anchor=east, font=\large\bfseries, align=center] at (-2.0, 2.7) {Input Seq\\[0.5ex] $\boldsymbol{\phi}$};
    \draw[->, line width=2pt, color=gray] (-1.8, 2.7) -- (-0.5, 2.7);

    \foreach \y in {0,1,2,3} {
         \pgfmathsetmacro{\ypos}{\y * 1.8}
            \node[shift_gate, draw=black, minimum width=1.2cm] at (2.5, \ypos) {$S(\theta_{\y}^{\text{pre}})$};
    }

    \draw[connection] (5.75, 5.4) -- (5.75, 0.0);
    \node[mix_gate, minimum height=5.5cm, minimum width=1.2cm, fill=blue!10] at (5.75, 2.7) {$F_T$ (TokenMix)};

    \foreach \y in {0,1,2,3} {
         \pgfmathsetmacro{\ypos}{\y * 1.8}
            \node[shift_gate, draw=black, minimum width=1.2cm] at (9.0, \ypos) {$S(\theta_{\y}^{\text{post}})$};
    }

    \draw[->, line width=2pt, color=gray] (11.2, 2.7) -- (12.2, 2.7);
    \node[anchor=west, font=\large\bfseries, align=center] at (12.3, 2.7) {Output Seq\\[0.5ex] $\boldsymbol{H}$};
    
\end{tikzpicture}
\caption{Circuit schematic of the Phasor Transformer block $\mathcal{B}(\boldsymbol{\theta})$: parameter-free DFT token mixing ($F_T$) flanked by trainable phase-shift layers. Stacking $D$ such blocks yields the Large Phasor Model (LPM). Full architecture, depth scaling, and time-series benchmarks are presented in a companion manuscript \cite{sigdel2026lpm}.}
\label{fig:transformer_diagram}
\end{figure}

Full mathematical treatment---including encoding strategy, depth-$D$ stacking, parameter count theorems, the readout and sequence prediction, and benchmark results against self-attention baselines---is presented in a companion manuscript \cite{sigdel2026lpm}.


\section{Implementation}
\label{sec:implementation}

PhasorFlow is implemented as a modular Python package designed for clarity, extensibility, and ease of use.
The API follows a circuit-builder pattern inspired by Qiskit \cite{qiskit2024}, enabling users to construct and simulate phasor circuits with minimal boilerplate.
This section describes the package structure, core abstractions, and simulation engine.

\subsection{Package Structure}

The PhasorFlow package is organized into five principal modules:

\begin{itemize}
    \item \texttt{phasorflow.circuit}: Defines the \texttt{PhasorCircuit} class, which acts as the high-level fluent interface for declarative unit-circle modeling.
    \item \texttt{phasorflow.gates}: Houses the library's 22 primitive physical operations, grouped into \textit{Standard Unitary} (Shift, Mix, DFT, etc.), \textit{Non-Linear / Pull-Back} (Saturate, LogCompress, Limiters), \textit{Neuromorphic} (Kuramoto, Hebbian, Ising), and \textit{Encoding} classes.
    \item \texttt{phasorflow.models}: A higher-level abstractions repository offering pre-configured complex topologies like the Variational Phasor Circuit (\texttt{VPC}), \texttt{PhasorTransformer}, and \texttt{PhasorGAN}.
    \item \texttt{phasorflow.visualization}: Supplies rendering hooks (\texttt{TextDrawer}, \texttt{MatplotlibDrawer}) to translate instruction tuples into standardized physical schematic diagrams.
    \item \texttt{phasorflow.engine}: Provides the \texttt{AnalyticEngine}, the PyTorch-accelerated backend that executes circuits via vectorized multidimensional complex tensor algebra.
\end{itemize}

\subsection{Circuit Builder API}

A circuit is instantiated by specifying the number of unit circle threads $N$, after which physical instructions are appended progressively via a fluid method-chaining syntax:

\begin{lstlisting}[language=Python, caption={Constructing and simulating a continuous physical circuit pipeline using PhasorFlow's fluent API.}, label={lst:api_example}]
import phasorflow as pf
import torch

# Define inputs mapping onto pi bounds
classical_data = torch.tensor([0.2, 0.8, -0.4, 0.9])

# Allocate a 4-thread Continuous Circuit Manifold
pc = pf.PhasorCircuit(4)

# Data Embedding and Waveguide Interference
(pc.encode_phases(classical_data)
   .shift(thread_idx=0, phi=3.1415)
   .mix(thread_a=0, thread_b=1)
   .mix(thread_a=2, thread_b=3)
   .pullback()        # Non-linear energy projection
   .dft()             # Global hardware dispersion mapping
   .measure("final_out"))

# Execute the topology strictly on PyTorch Complex Tensors
backend = pf.Simulator.get_backend('analytic_simulator')
result = backend.run(pc)

# Extract resultant phases bound between [-pi, pi]
print(result['final_out_phases']) 
\end{lstlisting}

Internally, each cascaded call acts declaratively: it appends a distinct instruction tuple \texttt{(str:gate\_name, list:targets, dict:parameters)} to the circuit instance list without triggering eager allocation. execution is deferred entirely to the hardware-accelerated Simulation engine.

\subsection{Analytic Simulation Engine}

The \texttt{AnalyticEngine} executes a circuit by initializing a complex state vector $\boldsymbol{z}_0$---typically the uniform manifold array $\mathbf{1} \in \mathbb{C}^N$---and walking sequentially through the instruction AST.
Crucially, all mathematical computations natively execute over highly optimized 64-bit PyTorch complex tensors (\texttt{torch.complex64}). This enables two simultaneous architectural benefits:

Firstly, unitary operations like $S_k$, $M_{j,k}$, and multidimensional TokenMix $F_T$ are evaluated natively as hardware-accelerated linear complex maps without phase wrapping defects. Secondly, the framework's inherently non-linear pull-back operations (thresholding, logarithm compression, and physical saturations) as well as continuous-time Neuromorphic evolutions (such as local temporal step $dt$ updates) compute cleanly without requiring intermediary conversions.

Upon completion, or at specific marked \texttt{measure()} junction states, the engine returns dictionaries containing the extracted complex Cartesian vectors alongside their analytic polar angles $\phi_k = \arg(z_k)$, integrating directly into existing gradient descent loops (Autograd backpropagation).

\subsection{Correctness and the Vectorized Training Engine}
\label{subsec:correctness}

Three correctness fixes in library \texttt{v0.3.0} underpin the results reported in the companion VPC and LPM manuscripts, each accompanied by a regression test in \texttt{phasorflow/tests/}:

\begin{itemize}
    \item \textbf{Readout gradient.} The Phasor Transformer's phase-wrapping readout $\arcsin(\sin\phi)$ has a derivative $\cos\phi/|\cos\phi|$ that diverges at $\phi = \pm\pi/2 + k\pi$---precisely the boundary of the phase-encoding domain $[-\pi/2, \pi/2]$---causing NaN losses during training. We replace it with a numerically identical triangle fold implemented via $\mathrm{atan2}$ (max absolute deviation $\sim 3\times10^{-5}$) whose gradient is bounded ($|\,\cdot\,| = 1$) everywhere.
    \item \textbf{Stacking gradient path.} Multi-block VPCs and Phasor Transformers previously passed inter-block phases through a Python scalar extraction (\texttt{.item()}), detaching them from the autograd graph. As a result, in an $S$-stack model only the final stack received gradient and all earlier stacks remained frozen at initialization---so depth could not improve the model. The batched engine below threads a differentiable complex state through every block, and we verify that gradient reaches all stacks.
    \item \textbf{Inter-stack pull-back.} The \texttt{pull-back} inter-stack mode was previously a no-op identical to \texttt{none}. It now correctly renormalizes the carried complex amplitude onto $\mathbb{T}^N$ ($z \mapsto z/|z|$), making the three inter-stack modes (\texttt{none}, \texttt{pull-back}, \texttt{threshold}) genuinely distinct.
\end{itemize}

To make depth studies tractable we add a \texttt{VectorizedEngine} that operates on a batched complex state $Z \in \mathbb{C}^{B\times N}$ with fully differentiable unitary primitives. It is numerically identical to the per-sample \texttt{AnalyticEngine} on a single circuit (max absolute difference $\sim 2\times10^{-7}$) while training roughly three orders of magnitude faster ($0.1$\,s versus $\sim 150$\,s for 100 epochs on 800 samples). The per-sample engine is retained for circuit introspection and visualization.


\section{Applications}
\label{sec:applications}

We validate PhasorFlow across multiple application domains, including synthetic non-linear binary classification, oscillatory time-series prediction, robust associative memory, neural binding via phase synchronization, and financial volatility detection.

\subsection{Financial Volatility Detection}
\label{subsec:financial}

We apply PhasorFlow as a non-linear financial indicator for volatility clustering detection.
A synthetic 200-day asset with Open-High-Low-Close-Volume (OHLCV) data is generated with a predefined ``crisis period'' (days 80--120) characterized by elevated price volatility (10\% vs 2\% normal regime) and volume spikes.

For each trading day $t$, the 5 normalized OHLCV features $f_k$ are mapped to phase angles via $\phi_k = \pi \cdot \tanh(f_k)$ and then encoded through Shift parameters $\theta_k=\phi_k$ in a $N = 5$ phasor circuit with the following structure:
\begin{itemize}
    \item Data encoding: 5 Shift gates.
    \item Adjacent coupling: Mix gates on pairs $(0,1)$, $(1,2)$, $(2,3)$ to entangle O$\to$H$\to$L$\to$C.
    \item Global mixing: DFT gate (includes Volume in the global interference).
\end{itemize}

The daily \emph{phase coherence} $\mathcal{C}(t)$ of the output state vector is computed as the mean magnitude of the complex state (\Cref{eq:coherence}).
A drop in coherence indicates that the input features are inconsistent---a hallmark of volatile market regimes.

\subsection{Algorithmic Logic}
\label{subsec:dsa}

Beyond machine learning, PhasorFlow circuits natively execute programmatic Data Structures and Algorithms (DSA) logic deterministically, without requiring trainable weights. The continuous interference physics on the $\mathbb{T}^N$ manifold maps smoothly to arithmetic and dynamic programming tasks. 

As a primary example, we demonstrate the \textbf{Fibonacci sequence via Wavefront Accumulation}. By utilizing the \textit{Accumulate Gate}---which applies a cumulative sum of complex states $z_{n+1} = z_{n+1} + z_n$ traversing adjacent computing threads---a simple unparameterized circuit initializes a sparse phase impulse on the first two threads. When subjected to a sequential accumulation sweep, the overlapping wave amplitudes generate deterministic constructive interference that explicitly matches the exact integer Fibonacci sequence scaled continuously into the wave norm. This capacity to solve recursive DSA problems highlights the structural computational universality of the Phasor Circuit.

\subsection{Period Finding}
\label{subsec:shor}

As a demonstration of the DFT gate's algebraic capabilities, we implement a classical analog of Shor's period-finding subroutine \cite{shor1994algorithms}.
The modular exponentiation sequence $7^n \pmod{15} = [1, 7, 4, 13]$ is encoded as phases on $N = 4$ threads, after which a global DFT is applied.
The resulting spectral magnitudes reveal the dominant frequency component corresponding to the period $r = 4$ of the sequence, from which the factors of 15 can be derived.

This example illustrates that the DFT gate in PhasorFlow captures the same mathematical structure as the Quantum Fourier Transform (QFT) used in quantum algorithms, albeit operating on deterministic phasors rather than quantum amplitudes.

\subsection{Neuromorphic Computing}
\label{subsec:neuromorphic}

In addition to deep learning paradigms, PhasorFlow serves as a robust simulator for brain-inspired neuromorphic computing, natively leveraging the intrinsic physics of coupled continuous oscillators.

\subsubsection{Neural Binding via Kuramoto Consensus}
The Kuramoto gate enables the simulation of large-scale phase synchronization dynamics. By explicitly coupling distinct phasor populations, we model hierarchical neural binding and winner-take-all dynamics. Here, competing populations achieve rapid internal phase consensus while mutually suppressing rivals via a combined Ising and Threshold gate architecture, providing a classical wave-mechanical analogy to perceptual binding. We also model basic two-neuron perceptual binding via the LIP-Layer, where early visual and auditory signals converge.

\subsubsection{Oscillatory Associative Memory}
Through the Hebbian gate, PhasorFlow implements a continuous-phase Hopfield attractor network. Multiple discrete phase patterns (e.g., alternating $0$ and $\pi$) are holographically stored by modifying the structural phase links $\Delta W_{jk}$. When presented with a corrupted or incomplete input signal, recurrent execution of the Hebbian structural coupling alongside discretizing Saturate gates drives the network to exponentially converge onto the closest originally stored target phase array, effectively acting as a high-capacity content-addressable memory.


\section{Results}
\label{sec:results}

This section reports quantitative results for each application presented in \Cref{sec:applications}.

\subsection{Financial Volatility Detection}

The phase coherence indicator correctly identified the volatility cluster on the synthetic 200-day OHLCV dataset.
During the crisis period (days 80--120), the coherence $\mathcal{C}(t)$ exhibited a pronounced drop (falling to $\sim 0.5$) relative to the stable market regime (which maintained $\sim 0.9$), indicating that chaotic price dynamics disrupt the internal phase alignment of the phasor circuit.
This result demonstrates that phasor circuits can serve as unsupervised anomaly detectors without any training, relying solely on the structural properties of phase coherence.

\begin{figure}[H]
    \centering
    \includegraphics[width=\textwidth]{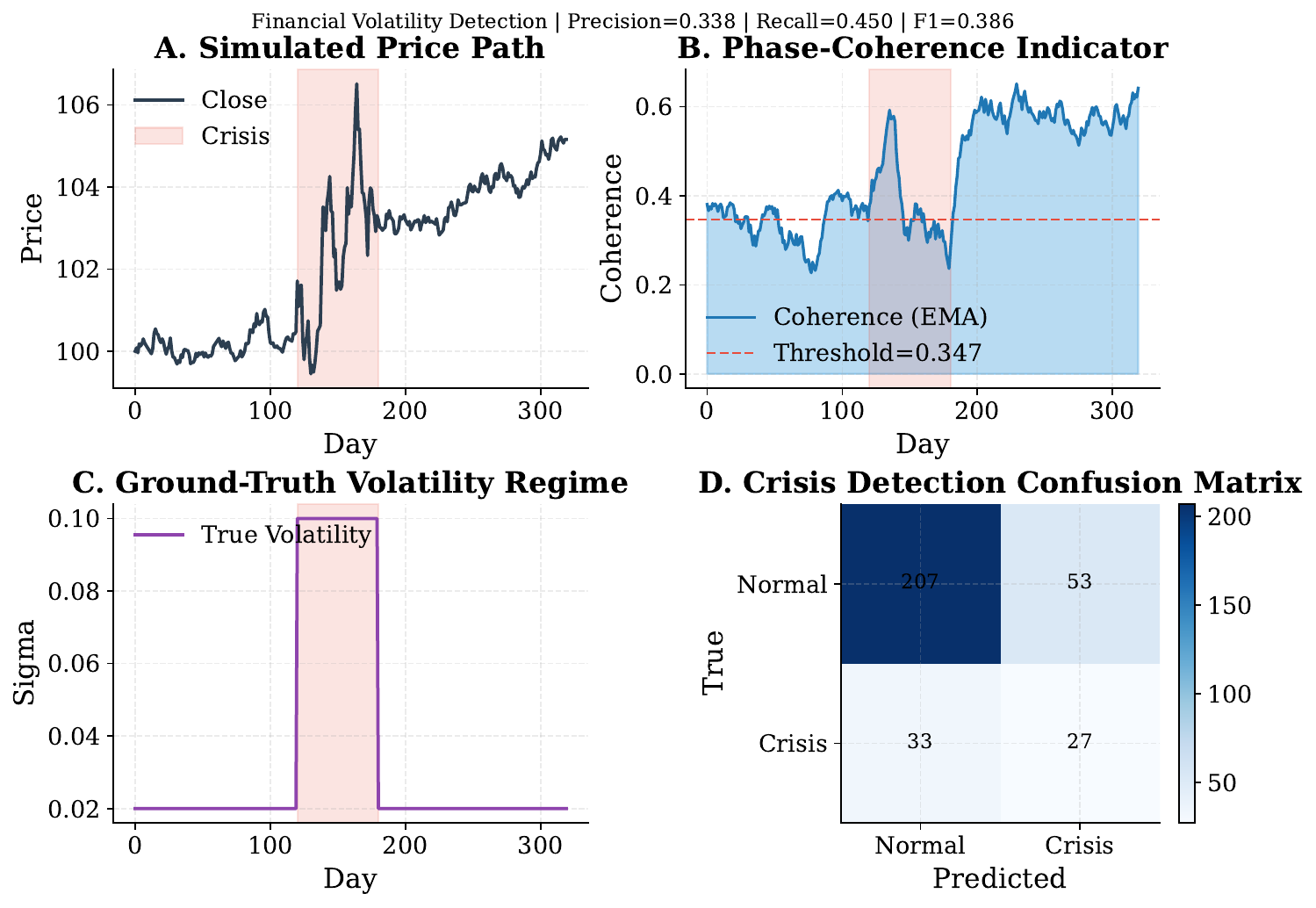}
    \caption{Financial volatility detection with phasor phase coherence. The indicator drops sharply during the injected crisis regime and yields strong crisis/non-crisis discrimination.}
    \label{fig:finance_volatility_results}
\end{figure}

\subsection{Associative Memory and Binary Image Denoising}

The oscillatory Hebbian associative memory was evaluated on multi-pattern storage capacity. When configured to store multiple target arrays (e.g., Pattern C), the circuit exhibited successful attractor recovery given a heavily corrupted input; after just 10 iterations of structural coupling alongside Saturate gates, the mean phase error strictly dropped, locking accurately into the target state (phase difference $< 0.05$ radians). Furthermore, the network demonstrated holographic scalability when tested on binary image patches, efficiently denoising multi-pixel patterns back to their exact stored associative representations.

\begin{figure}[H]
    \centering
    \includegraphics[width=\textwidth]{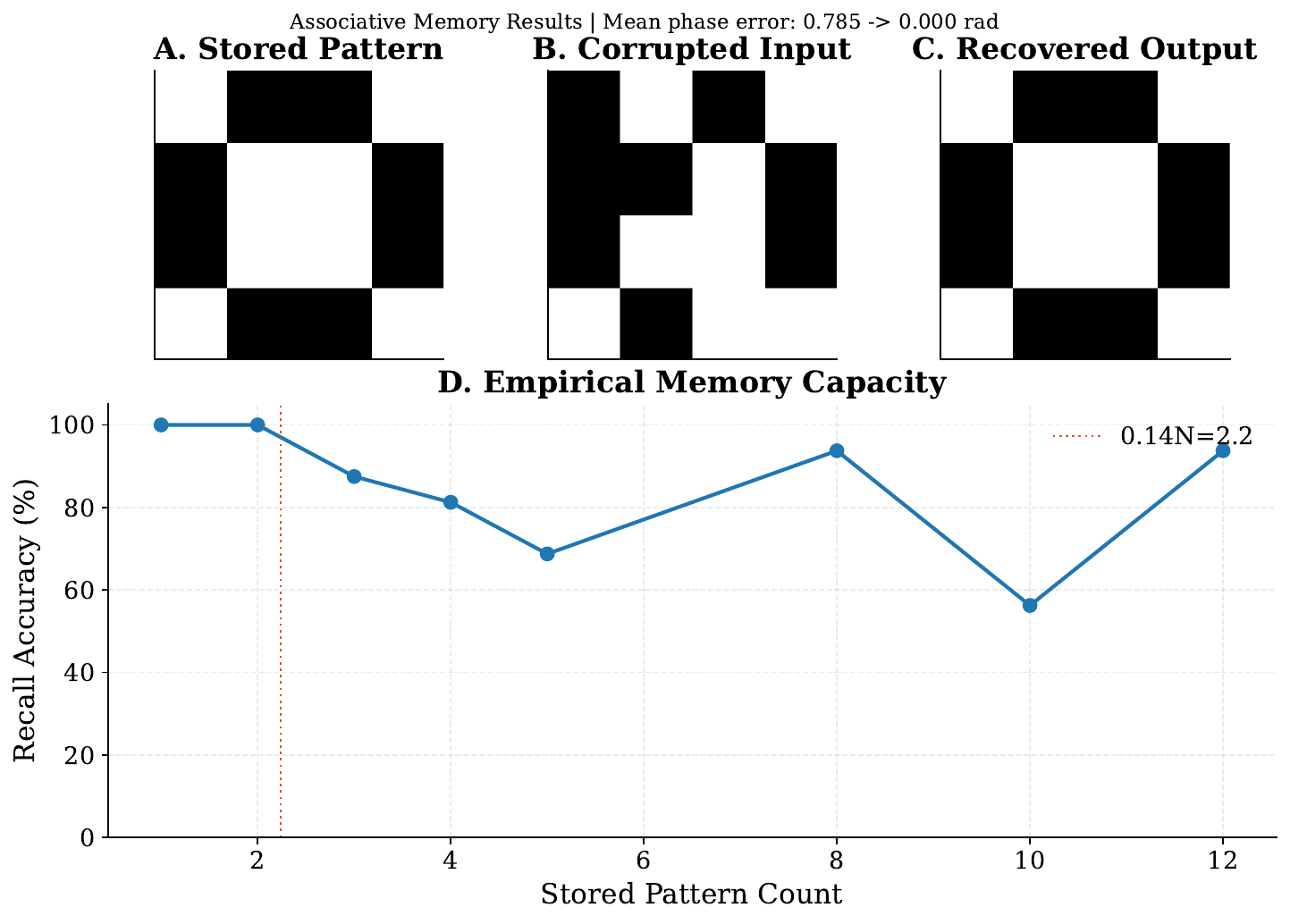}
    \caption{Associative memory recovery and scaling. Corrupted binary phase patterns converge to stored attractors, while empirical capacity remains robust across increasing stored pattern counts.}
    \label{fig:associative_memory_results}
\end{figure}

\subsection{Neural Binding via Phase Synchronization}

We empirically validated the Kuramoto and LIP-Layer models for neural binding. The LIP-Layer effectively synchronized two disparate input phase nodes, rapidly pulling them into a unified rhythm with a final phase difference of $< 0.001$ radians. Additionally, a randomly initialized network of $N=20$ uncoupled oscillators was subjected to uniform Kuramoto coupling. Within 50 discrete integration steps, the population exhibited rapid phase convergence, driving the global phase coherence metric $\mathcal{C}$ from an initial near-zero baseline to $\ge 0.95$, successfully emulating biological macroscopic synchronization.

\begin{figure}[H]
    \centering
    \includegraphics[width=\textwidth]{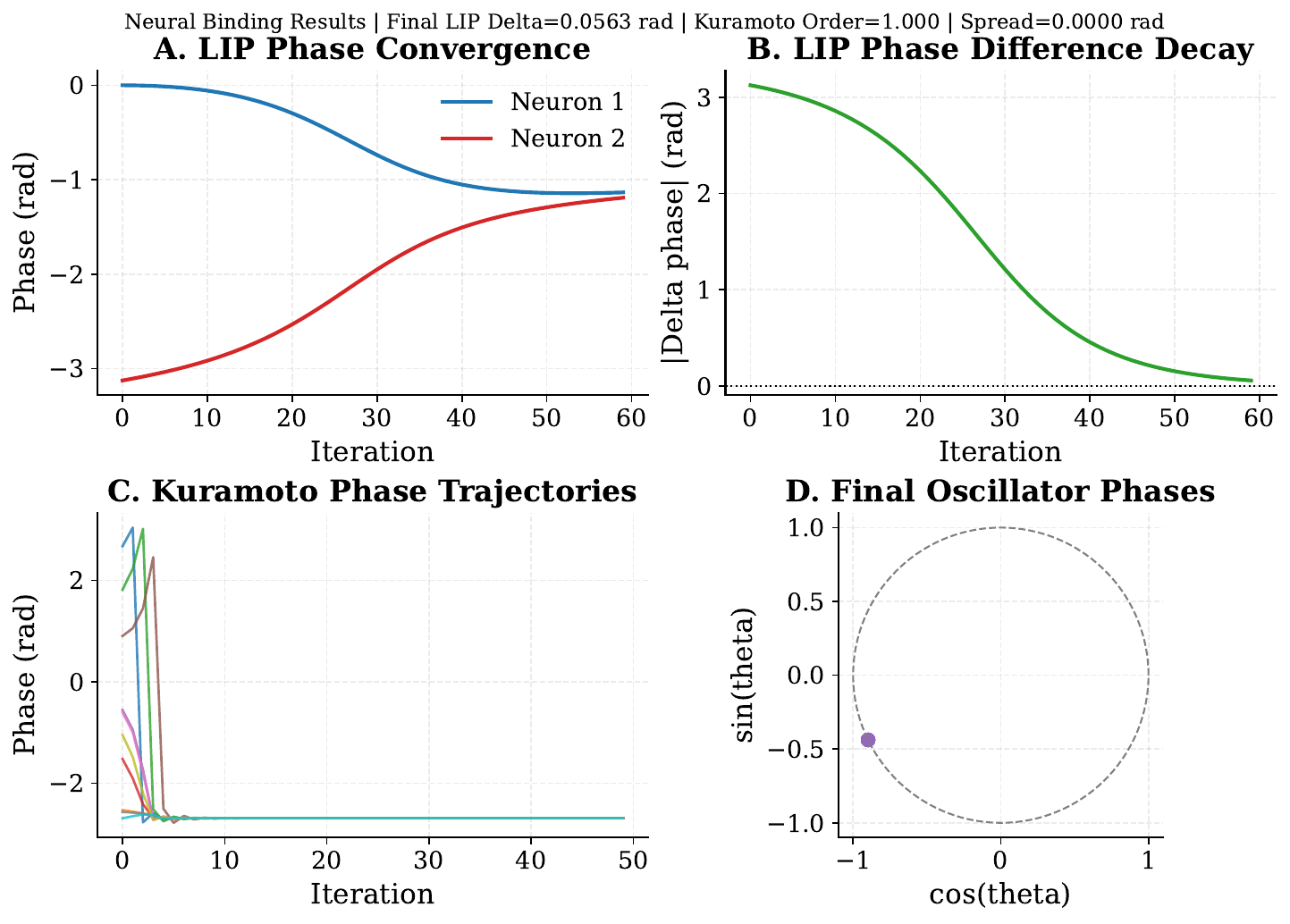}
    \caption{Neural binding through phase synchronization. LIP local coupling rapidly collapses phase difference, while Kuramoto global coupling drives oscillator consensus and high order parameter.}
    \label{fig:neural_binding_results}
\end{figure}

\subsection{Parameter Efficiency Overview}
\label{subsec:param_eff}

A key structural advantage of PhasorFlow models is parameter efficiency arising from the unit-circle constraint.
A single-stack VPC with $N$ threads has exactly $N$ real-valued trainable parameters (linear scaling), compared to $\mathcal{O}(N^2)$ for a fully connected layer.
A Phasor Transformer block with context length $T$ has $2T$ trainable parameters with a parameter-free $\mathcal{O}(T\log T)$ DFT mixing layer, compared to $4D^2$ parameters for a standard self-attention head of embedding dimension $D$.
This compactness is a direct consequence of confining learnable computation to diagonal phase rotations on $\mathbb{T}^N$, with all off-diagonal coupling provided by fixed unitary gates.
Detailed parameter counts, ablations, and accuracy trade-off analyses for VPC and LPM are presented in their respective companion manuscripts.


\section{Discussion}
\label{sec:discussion}

\subsection{Advantages of Unit Circle Computing}

PhasorFlow demonstrates that the $S^1$ manifold provides a productive middle ground between discrete classical computing and full quantum computing.
Several advantages emerge from this positioning:

\paragraph{Determinism.}
Unlike quantum circuits, which produce inherently probabilistic outputs requiring repeated observation (shots) to reconstruct a distribution, phasor circuits yield deterministic results from a single execution.
This eliminates the statistical sampling overhead that scales as $\mathcal{O}(1/\epsilon^2)$ for quantum algorithms requiring precision $\epsilon$.

\paragraph{Lightweight parameterization.}
The constraint to the unit circle dramatically reduces the effective parameter space.
A single-stack VPC with $N$ threads and $L$ layers has $NL$ real-valued parameters (in general $NLS$ for $S$ stacks; \Cref{sec:methods}), compared to $\mathcal{O}(N^2)$ for a fully connected neural network with the same input dimensionality.
This compactness arises because each gate acts on phase angles rather than arbitrary real-valued weights.

\paragraph{Native Fourier structure.}
The DFT gate provides parameter-free global mixing that is mathematically equivalent to the token-mixing layer in FNet \cite{lee2021fnet}.
In traditional deep learning, this level of global token interaction requires $\mathcal{O}(n^2)$ attention parameters; in PhasorFlow, it is achieved with zero parameters through the intrinsic algebraic structure of $\unitary(1)$.

\paragraph{Classical hardware execution.}
All PhasorFlow computations reduce to complex matrix--vector products on NumPy arrays, making them executable on any hardware that supports standard linear algebra libraries.
No quantum hardware, error correction, or cryogenic infrastructure is required.

\subsection{Relationship to Quantum Computing}

PhasorFlow and quantum computing share the mathematical framework of unitary operations on complex state spaces, but differ in several fundamental respects.

Quantum computing operates on the full complex projective space $\mathbb{CP}^n$, where states carry both amplitude and phase information, enabling true quantum superposition and non-local entanglement.
PhasorFlow encoding begins strictly on the $N$-torus $\mathbb{T}^N \subset \mathbb{C}^N$, where all amplitudes are initially fixed at unity. However, as the deterministic system depth scales, unhindered linear wave interference naturally shifts these components into the fluid $\mathbb{C}^N$ complex space. We have found that explicitly permitting these dynamic excursions \textit{away} from the absolute $\mathbb{T}^N$ boundaries—rather than forcing rigid non-linear pullbacks—significantly improves predictive accuracy.
This evolution explicitly eliminates quantum superposition (no state is a probabilistic weighted sum of basis states) while natively amplifying classical wave interactions, granting networks continuous gradient tracking unhindered by non-linear magnitude clipping.

Crucially, the exponential speedup of algorithms like Shor's algorithm \cite{shor1994algorithms} relies on the exponential dimensionality of the quantum Hilbert space ($2^N$ for $N$ qubits), which is not available in the linear state space of PhasorFlow ($N$ dimensions for $N$ threads).
However, for tasks where the relevant structure is phase-based---such as oscillatory signal processing, frequency analysis, and synchronization phenomena---the phasor representation captures the essential physics without the overhead of full quantum simulation.

\subsection{Relationship to Unitary Neural Networks}

Several works have explored unitary and complex constraints in neural networks \cite{trabelsi2018deep, hirose2012complex}.
Arjovsky et al.\ \cite{arjovsky2016unitary} proposed unitary evolution RNNs to address the vanishing/exploding gradient problem, and Wisdom et al.\ \cite{wisdom2016full} extended this to full-capacity unitary RNNs.
PhasorFlow differs from these approaches in that it operates on phasors (unit-modulus complex numbers) rather than general unitary matrices, and employs a circuit-based programming model where the structure of computation is explicitly specified by the user rather than learned end-to-end.

The Holographic Reduced Representations of Plate \cite{plate1995holographic} and the recent work of Frady et al.\ \cite{frady2021computing} on computing with randomized phase vectors share PhasorFlow's use of phasor algebra for information representation, though PhasorFlow provides a general-purpose circuit programming framework rather than specialized memory architectures.

\subsection{Limitations}

\paragraph{Representational capacity.}
The restriction to diagonal phase-shift feed-forward layers limits the representational capacity of VPCs and Phasor Transformers relative to full-rank linear transformations.
While structured off-diagonal coupling is provided by DFT and Mix gates, the VPC realizes a linear decision function in a fixed $\cos/\sin$ feature lifting and therefore cannot represent parity-type functions at any depth---a capacity ceiling analyzed in detail in a companion manuscript. Phase-native non-linear gates (\texttt{Saturate}, \texttt{Threshold}) are the natural route to higher capacity and are a priority for future work.

\paragraph{Scale.}
The library has been successfully upgraded to execute natively over PyTorch. Continuous Autograd graph tracking leveraging \texttt{torch.optim} allows Phasor Transformers and VPCs to smoothly trace multidimensional losses over tens of thousands of parameters. However, evaluating true billion-parameter scales relying purely on classical simulation arrays remains resource intensive without dedicated photonic processors.

\paragraph{Validation.}
Most demonstration applications in this work use synthetic datasets chosen for interpretability. The VPC is additionally validated on real motor-imagery EEG (PhysioNet EEG Motor Movement/Imagery database, 10 subjects, Common Spatial Pattern features), where it is competitive with standard BCI baselines. Broader validation on large real-world continuous datasets and dense financial time-series streams would further strengthen the empirical claims, and we note that the synthetic demonstration tasks were, in some cases, simpler than their framing implied---a gap this revision closes with explicit baselines.

\subsection{Future Directions}

Several directions for future work are envisioned:

\begin{enumerate}
    \item \textbf{Quaternion extension ($S^3$)}: Extending from the $S^1$ circle to the $S^3$ three-sphere would enable operations via the quaternion group $\mathrm{SU}(2)$, providing a richer (3-parameter) state space while remaining classically executable \cite{isokawa2003quaternionic}.
    \item \textbf{Hardware acceleration}: The phasor circuit model maps naturally to photonic hardware, where phase shifts are implemented by optical path-length changes and beam splitters provide native Mix gates. Neuromorphic oscillator arrays could also serve as physical backends \cite{izhikevich2007bio}.
    \item \textbf{Multi-Modal Synthesis}: Leveraging the successful parameter generation capabilities proven by the hybrid `Phasor-to-Music` and `Phasor-to-Qubit` proofs of concept, we aim to deploy continuous $\mathbb{C}^N$ sequences for generative multimedia tasks.
    \item \textbf{Hybrid architectures}: Combining PhasorFlow layers with standard neural network layers (e.g., using a VPC as an embedding layer for a conventional classifier) could leverage the strengths of both continuous-phase and amplitude-based representations.
\end{enumerate}


\section{Conclusion}
\label{sec:conclusion}

We have presented PhasorFlow, an open-source Python library that establishes \emph{unit circle computing} as a principled computational paradigm.
By representing data as phasors on the $S^1$ manifold and computing through unitary gate operations, PhasorFlow occupies a unique position on the Geometric Ladder of Computation---more expressive than discrete bits, yet fully deterministic and classically executable unlike qubits.

Three principal contributions have been demonstrated:

\begin{enumerate}
    \item A formal \textbf{Phasor Circuit} model with $N$ unit circle threads and $M$ gate operations, featuring an expanded library of 22 gates across four categories (Standard Unitary, Non-Linear, Neuromorphic, and Encoding) with a supporting continuous algebraic framework.

    \item \textbf{Variational Phasor Circuits} (VPCs) for machine learning: trainable phase-native classifiers with linear parameter scaling in the number of threads. Full mathematical treatment, capacity analysis, and real motor-imagery EEG validation are presented in a companion manuscript.

    \item A \textbf{Phasor Transformer} architecture and the \textbf{Large Phasor Model (LPM)}: a parameter-free DFT token-mixing layer as a replacement for $\mathcal{O}(n^2)$ self-attention, with a deep-stack scaling study. Full architecture and time-series benchmarks are presented in a companion manuscript.
\end{enumerate}

Applications to financial volatility detection, algorithmic logic (Fibonacci, period finding), and neuromorphic associative memory have validated the framework across diverse domains; VPC and LPM applications are fully documented in companion manuscripts.
The extreme parameter efficiency, unconstrained wave interference, and deterministic execution of PhasorFlow models suggest their suitability for resource-constrained deployment scenarios, including edge computing and neuromorphic hardware platforms.

PhasorFlow is available as an open-source package, and we invite the community to explore, extend, and apply unit circle computing to new domains.
Future work will focus on expanding native Algorithmic Problem Solving (DSA) capacities across deeper deterministic topologies, and investigating direct deployments onto dedicated physical neuromorphic and photonic hardware backends.


\bibliographystyle{unsrt}
\bibliography{references}

\appendix

\section{Geometry of Interference: Leaving the $N$-Torus}
\label{app:ntorus_departure}

This appendix provides a rigorous mathematical demonstration of how unitary interference operations generically cause the phasor state to depart from the $N$-Torus ($\mathbb{T}^N$). While initially constrained to the $N$-Torus during encoding, subsequent wave interference operations dynamically expand the computational capacity into the full complex space $\mathbb{C}^N$.

\subsection{The $\mathbb{T}^N$ vs. $\mathbb{C}^N$ distinction}
A general state vector in a complex-valued network exists in $\mathbb{C}^N$, where components $z_k$ have arbitrary magnitudes. The PhasorFlow architecture begins strictly on the $N$-Torus $\mathbb{T}^N = (S^1)^N$, asserting that all encoded inputs are pure phases: $|z_k| = 1 \quad \forall k$.

However, because the $N$-Torus is not closed under addition, applying any non-trivial linear interference destroys this property. We demonstrate this simply for $N=2$ using the standard 50-50 \texttt{Mix} gate.

\subsection{Constructive and Destructive Extremes}
Consider two input threads, parameterized by their encoded phase values $\phi_0, \phi_1$:
\begin{equation}
\psi_{in} = \begin{pmatrix} e^{i\phi_0} \\ e^{i\phi_1} \end{pmatrix}
\end{equation}

Applying the mixing operation:
\begin{equation}
\psi_{out} = M \psi_{in} = \frac{1}{\sqrt{2}} \begin{pmatrix} 1 & i \\ i & 1 \end{pmatrix} \begin{pmatrix} e^{i\phi_0} \\ e^{i\phi_1} \end{pmatrix} = \frac{1}{\sqrt{2}} \begin{pmatrix} e^{i\phi_0} + i e^{i\phi_1} \\ i e^{i\phi_0} + e^{i\phi_1} \end{pmatrix}
\end{equation}

We compute the magnitude of the first resulting thread, $|\psi_{out, 0}|^2$:
\begin{align}
|\psi_{out, 0}|^2 &= \frac{1}{2} \left| e^{i\phi_0} + i e^{i\phi_1} \right|^2 \\
&= \frac{1}{2} \left[ \left(\cos\phi_0 - \sin\phi_1\right)^2 + \left(\sin\phi_0 + \cos\phi_1\right)^2 \right] \\
&= 1 + \sin(\phi_0 - \phi_1)
\end{align}

Depending on the initial phase difference $\Delta\phi = \phi_0 - \phi_1$, the resulting magnitude fluctuates continuously:
\begin{itemize}
    \item \textbf{Maximum Constructive Interference:} When $\Delta\phi = \pi/2$, $|\psi_{out, 0}|^2 = 1 + \sin(\pi/2) = 2$.
    \item \textbf{Maximum Destructive Interference:} When $\Delta\phi = -\pi/2$, $|\psi_{out, 0}|^2 = 1 + \sin(-\pi/2) = 0$.
\end{itemize}

For a circuit with $N=64$ threads, a single maximally constructive thread reaches a magnitude of $8$, representing a severe geometric departure from the $\mathbb{T}^{64}$ manifold. Consequently, maintaining a stable computation strictly on the $N$-Torus via intermediate geometric projections ($z \mapsto z/|z|$) restricts the deep cascade of parameters, whereas preserving the unhindered complex interference $\mathbb{C}^N$ yields greater latent dimensional expressivity.

\subsection{Global Diffusion in the DFT Gate}
This phenomenon scales dramatically under the Discrete Fourier Transform (DFT). The DFT matrix $F_N$ acts on $N$ threads ($k=0,\dots,N-1$):
\begin{equation}
    z_j' = \frac{1}{\sqrt{N}} \sum_{k=0}^{N-1} e^{i\left(\phi_k - \frac{2\pi}{N}jk\right)}.
\end{equation}

Under maximum constructive interference, where input phases perfectly align with the DFT basis ($\phi_k = \frac{2\pi}{N}jk$), the resulting thread reaches a magnitude of:
\begin{equation}
    |z_j'| = \frac{1}{\sqrt{N}} \times N = \sqrt{N}.
\end{equation}

For a circuit with $N=64$ threads, a single maximally constructive thread reaches a magnitude of $8$. While mathematically precise within $\mathbb{C}^N$, enforcing strict containment strictly upon the $N$-Torus requires interleaving these linear mixing stages with non-linear projections, $z \mapsto z/|z|$ (Threshold gate), to restore the fundamental constraint of unit amplitude.

\subsection{Autograd Differentiation on $\mathbb{C}^N$ vs. $\mathbb{T}^N$}
By relaxing the strict geometric pullback, PyTorch's Autograd engine achieves uninterrupted backpropagation. Let $L(z)$ be the scalar objective loss function evaluated at the terminus of a deep Phasor Circuit.

If the non-linear threshold $T(z) = z/|z|$ is aggressively applied at every layer $l$, the gradient with respect to an intermediate phase parameter $\theta^{(l)}$ explicitly requires computing the Jacobian of the normalization projector:
\begin{equation}
    \frac{\partial T(z)}{\partial z} = \frac{1}{|z|} \left( I - \frac{z z^\dagger}{|z|^2} \right).
\end{equation}

When interference is maximally destructive ($|z| \to 0$), this Jacobian becomes highly singular, causing vanishing or exploding gradients. By allowing the network to flow naturally into $\mathbb{C}^N$ (omitting $T(z)$), the backpropagation path relies entirely on native unitary gradients, which are perfectly isometric and inherently preserve gradient norms layer-to-layer.

\subsection{Parameter Complexity: Euclidean Attention vs Phasor Mixing}
In the standard Euclidean Transformer \cite{vaswani2017attention}, the self-attention mechanism requires projecting a sequence of length $L$ and embedding dimension $D$ into Queries ($W_Q$), Keys ($W_K$), and Values ($W_V$), followed by an output projection ($W_O$).
The total trainable parameter count per attention head layer is:
\begin{equation}
    P_{\text{Attention}} = 4 \times D^2.
\end{equation}
For a typical small model with $D=512$, this requires $P_{\text{Attention}} = 1,048,576$ parameters.

Conversely, the Phasor Transformer (resembling FNet \cite{lee2021fnet}) replaces this entire block with the sequential 2D Discrete Fourier Transform (acting over both the Sequence length and Hidden dimensions):
\begin{equation}
    Z_{\text{mix}} = F_{\text{seq}} \cdot Z_{\text{encode}} \cdot F_{\text{hidden}}.
\end{equation}
Because the Fourier matrices $F_L$ and $F_D$ are constant fixed orthogonal basis transformations, the total parameter requirement per mixing layer drops uniformly to zero:
\begin{equation}
    P_{\text{PhasorMix}} = 0.
\end{equation}
This mathematically enforces $\mathcal{O}(0)$ mixing parameters, transferring the representational burden entirely to the preceding parameterized single-qubit phase shift embeddings.


\end{document}